\documentclass{article}

\usepackage{microtype}
\usepackage{graphicx}
\usepackage{subfigure}
\usepackage{booktabs} %

\usepackage[allcolors=black]{hyperref}

\usepackage[accepted]{icml2021}

\usepackage{graphicx}
\usepackage{tikz}

\usepackage{pifont}

\usepackage{amsmath,amsfonts,bm}

\def\eqref#1{equation~\ref{#1}}

\def\1{\bm{1}}

\DeclareMathAlphabet{\mathsfit}{\encodingdefault}{\sfdefault}{m}{sl}
\SetMathAlphabet{\mathsfit}{bold}{\encodingdefault}{\sfdefault}{bx}{n}

\newcommand{\E}{\mathbb{E}}

\newcommand{\R}{\mathbb{R}}

\usepackage{natbib}
\usepackage{wrapfig}
\usepackage{booktabs}   
\usepackage{amssymb}
\hypersetup{colorlinks=true,linkcolor=blue,citecolor=blue}
\usepackage{url}
\usepackage{amsthm}
\usepackage{environ}

\usepackage{multirow}

\usepackage{xcolor}

\newcommand{\cmark}{\ding{51}}%
\newcommand{\xmark}{\ding{55}}%

\usepackage{tikz}

\newcommand{\dittotikz}{%
    \tikz{
        \draw [line width=0.12ex] (-0.2ex,0) -- +(0,0.8ex)
            (0.2ex,0) -- +(0,0.8ex);
        \draw [line width=0.08ex] (-0.6ex,0.4ex) -- +(-1.5em,0)
            (0.6ex,0.4ex) -- +(1.5em,0);
    }%
}

\renewcommand{\d}{{\bf{d}}}

\newcommand{\J}{{\bf{J}}}
\newcommand{\A}{\bf{A}}

\newcommand{\x}{{\bf{x}}}

\newcommand{\z}{{\bf{z}}}

\newcommand{\tx}{{\bf{\tilde x}}}
\renewcommand{\d}{{\rm{d}}}

\newcommand{\zz}{\mathbf{z}}
\newcommand{\xx}{\mathbf{x}}

\newcommand{\txx}{{\tilde{\mathbf{x}}}}
\newcommand{\tzz}{{\tilde{\mathbf{z}}}}

\newcommand{\Sp}{\mathbb{S}}

\newcommand{\X}{\mathcal{X}}
\newcommand{\Z}{\mathcal{Z}}

\newcommand*{\lunif}{\mathcal{L}_\mathsf{uni}}
\newcommand*{\lalign}{\mathcal{L}_\mathsf{align}}
\newcommand*{\lcontr}{\mathcal{L}_\mathsf{contr}}
\newcommand*{\ldeltaunif}{\mathcal{L}_{\delta\text{-}\mathsf{uni}}}
\newcommand*{\ldeltaalign}{\mathcal{L}_{\delta\text{-}\mathsf{align}}}
\newcommand*{\ldeltacontr}{\mathcal{L}_{\delta\text{-}\mathsf{contr}}}

\newcommand*{\lce}{\mathcal{L}_{\mathsf{CE}}}
\newcommand{\T}[0]{^{\mathsf{T}}}

\newcommand*{\iid}{\ifmmode \stext{i.i.d.} \else i.i.d.\@\xspace \fi}

\renewcommand{\eqref}[1]{(\ref{#1})}

\newcommand{\qh}{q_{\rm{h}}}

\newcommand{\expectunder}[2]{\underset{{#1}}{\mathbb{E}}\left[#2\right]}

\newtheorem{innercustomgeneric}{\customgenericname}
\providecommand{\customgenericname}{}
\newcommand{\newcustomtheorem}[2]{%
  \newenvironment{#1}[1]
  {%
   \renewcommand\customgenericname{#2}%
   \renewcommand\theinnercustomgeneric{##1}%
   \innercustomgeneric
  }
  {\endinnercustomgeneric}
}
\newcustomtheorem{customtheorem}{Theorem}
\newcustomtheorem{customcorollary}{Corollary}
\newcustomtheorem{customproposition}{Proposition}
\newcustomtheorem{customlemma}{Lemma}

\newtheorem{theorem}{Theorem}
\newtheorem{corollary}{Corollary} 
\newtheorem{proposition}{Proposition}
\newtheorem{lemma}{Lemma}
\theoremstyle{definition} \newtheorem{definition}{Definition}

\NewEnviron{equations}{%
\begin{equation}\begin{split}
  \BODY
\end{split}\end{equation}
}

\usepackage{mathtools}

\icmltitlerunning{Contrastive Learning Inverts the Data Generating Process}

\begin{document}

\twocolumn[
\icmltitle{Contrastive Learning Inverts the Data Generating Process}

\icmlsetsymbol{equal}{*}

\begin{icmlauthorlist}
\icmlauthor{Roland S. Zimmermann}{equal,tue,imprs}
\icmlauthor{Yash Sharma}{equal,tue,imprs}
\icmlauthor{Steffen Schneider}{equal,tue,imprs,epfl}
\icmlauthor{Matthias Bethge \textsuperscript{\textdagger}}{tue}
\icmlauthor{Wieland Brendel \textsuperscript{\textdagger}}{tue}
\end{icmlauthorlist}

\icmlaffiliation{tue}{University of T\"ubingen, T\"ubingen, Germany}
\icmlaffiliation{imprs}{IMPRS for Intelligent Systems, T\"ubingen, Germany}
\icmlaffiliation{epfl}{EPFL, Geneva, Switzerland}

\icmlcorrespondingauthor{Roland S. Zimmermann}{roland.zimmermann@uni-tuebingen.de}

\icmlkeywords{Machine Learning, ICML}

\vskip 0.3in
]

\printAffiliationsAndNotice{\icmlEqualContribution} %

\begin{abstract}
    Contrastive learning has recently seen tremendous success in self-supervised learning. So far, however, it is largely unclear why the learned representations generalize so effectively to a large variety of downstream tasks. We here prove that feedforward models trained with objectives belonging to the commonly used InfoNCE family learn to implicitly invert the underlying generative model of the observed data. While the proofs make certain statistical assumptions about the generative model, we observe empirically that our findings hold even if these assumptions are severely violated.
    Our theory highlights a fundamental connection between contrastive learning, generative modeling, and nonlinear independent component analysis, thereby furthering our understanding of the learned representations as well as providing a theoretical foundation to derive more effective contrastive losses.\footnote{Online version and code: \href{https://brendel-group.github.io/cl-ica/}{brendel-group.github.io/cl-ica/}} %
\end{abstract}

\section{Introduction}
    With the availability of large collections of unlabeled data, recent work has led to significant advances in self-supervised learning. In particular, contrastive methods have been tremendously successful in learning representations for visual and sequential data \citep{logeswaran2018efficient,wu2018unsupervised,oord2018representation,henaff2020data,tian2019contrastive,hjelm2018learning,bachman2019learning,he2019momentum,chen2020simple,schneider2019wav2vec,Baevski2020vqwav2vec,baevski2020wav2vec,ravanelli2020multi}.     %
    While a number of explanations have been provided as to why contrastive learning leads to such informative representations, existing theoretical predictions and empirical observations appear to be at odds with each other~\citep{tian2019contrastive,bachman2019learning,wu2020importance,saunshi2019theoretical}. 
    
    In a nutshell, contrastive methods aim to learn representations where related samples are aligned (positive pairs, e.g. augmentations of the same image), while unrelated samples are separated (negative pairs)~\citep{chen2020simple}.
    Intuitively, this leads to invariance to irrelevant details or transformations (by decreasing the distance between positive pairs), while preserving a sufficient amount of information about the input for solving downstream tasks (by increasing the distance between negative pairs)~\citep{tian2020makes}.
    This intuition has recently been made more precise by \cite{wang2020understanding}, showing that a commonly used contrastive loss from the InfoNCE family~\citep{Gutmann12JMLR, oord2018representation, chen2020simple} asymptotically converges to a sum of two losses: an \emph{alignment} loss that pulls together the representations of positive pairs, and a \emph{uniformity} loss that maximizes the entropy of the learned latent distribution.
    
    \begin{figure*}
        \centering
        \includegraphics[width=0.8\textwidth]{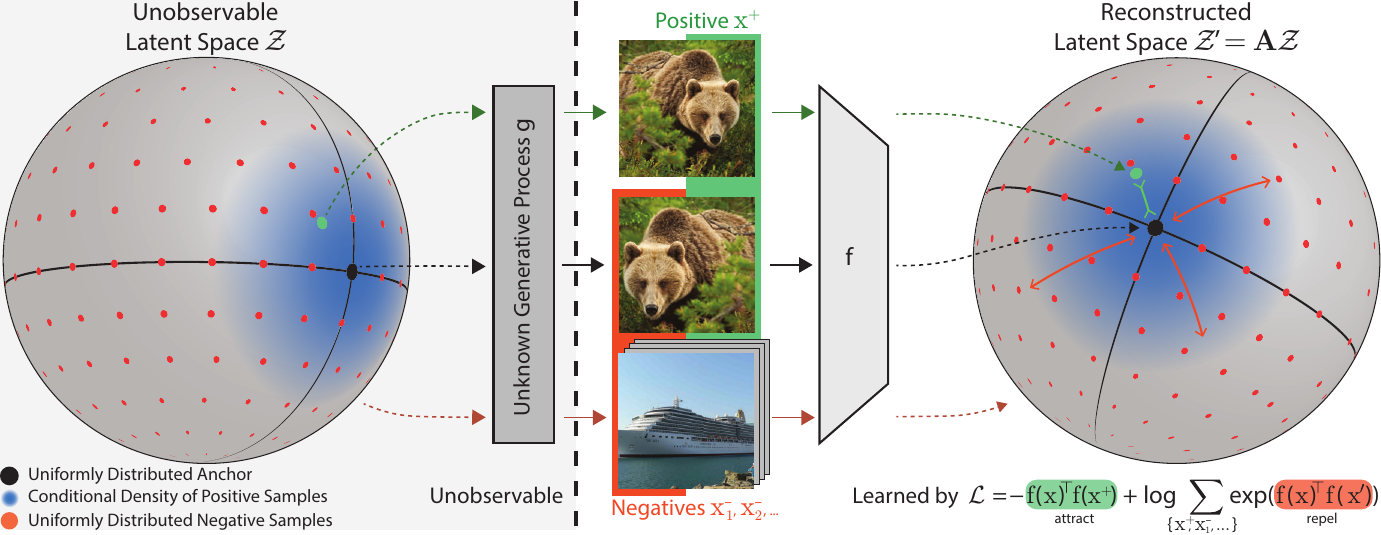}
        \caption{We analyze the setup of contrastive learning, in which a feature encoder $f$ is trained with the InfoNCE objective \citep{Gutmann12JMLR, oord2018representation, chen2020simple} using positive samples (green) and negative samples (orange). We assume the observations are generated by an (unknown) injective generative model $g$ that maps unobservable latent variables from a hypersphere to observations in another manifold. Under these assumptions, the feature encoder $f$ implictly learns to invert the ground-truth generative process $g$ up to linear transformations, i.e., $f = \mathbf{A} g^{-1}$ with an orthogonal matrix $\mathbf{A}$, if $f$ minimizes the InfoNCE objective.}
        \label{fig:header_figure}
    \end{figure*}
    
    We show that an encoder learned with a contrastive loss from the InfoNCE family can recover the true generative factors of variation (up to rotations) if the process that generated the data fulfills a few weak statistical assumptions. This theory bridges the gap between contrastive learning, nonlinear independent component analysis (ICA) and generative modeling (see Fig.~\ref{fig:header_figure}).
    Our theory reveals implicit assumptions encoded in the InfoNCE objective about the generative process underlying the data. If these assumptions are violated, we show a principled way of deriving alternative contrastive objectives based on assumptions regarding the positive pair distribution.
    We verify our theoretical findings with controlled experiments, providing evidence that our theory holds true in practice, even if the assumptions on the ground-truth generative model are partially violated. %
    
    To the best of our knowledge, our work is the first to analyze under what circumstances representation learning methods used in practice provably represent the data in terms of its underlying factors of variation. Our theoretical and empirical results suggest that the success of contrastive learning in many practical applications is due to an implicit and approximate inversion of the data generating process, which explains why the learned representations are useful in a wide range of downstream tasks.
    
    In summary, our contributions are:
    \begin{itemize}
        \item We establish a theoretical connection between the InfoNCE family of objectives, which is commonly used in self-supervised learning, and nonlinear ICA. We show that training with InfoNCE inverts the data generating process if certain statistical assumptions on the data generating process hold.
        \item We empirically verify our predictions when the assumed theoretical conditions are fulfilled. In addition, we show successful inversion of the data generating process even if these theoretical assumptions are partially violated. %
        \item We build on top of the CLEVR rendering pipeline~\citep{johnson2017clevr} to generate a more visually complex disentanglement benchmark, called \emph{3DIdent}, that contains hallmarks of natural environments (shadows, different lighting conditions, a 3D object, etc.). We demonstrate that a contrastive loss derived from our theoretical framework can identify the ground-truth factors of such complex, high-resolution images.
    \end{itemize}
    
\section{Related Work}
\paragraph{Contrastive Learning}
    Despite the success of contrastive learning (CL), our understanding of the learned representations remains limited, as existing theoretical explanations yield partially contradictory predictions. One way to theoretically motivate CL is to refer to the InfoMax principle \citep{linsker1988self}, which corresponds to maximizing the mutual information (MI) between different views \citep{oord2018representation, bachman2019learning, hjelm2018learning, chen2020simple, tian2020makes}. However, as optimizing a tighter bound on the MI can produce worse representations \citep{tschannen2019mutual}, it is not clear how accurate this motivation describes the behavior of CL.
    
    Another approach aims to explain the success by introducing latent classes \citep{saunshi2019theoretical}. While this theory has some appeal, there exists a gap between empirical observations and its predictions, e.g. the prediction that an excessive number of negative samples decreases performance does not corroborate with empirical results~\citep{wu2018unsupervised,tian2019contrastive,he2019momentum,chen2020simple}. However, recent work has suggested some empirical evidence for said theoretical prediction, namely, issues with the commonly used sampling strategy for negative samples, and have proposed ways to mitigate said issues as well~\citep{robinson2020contrastive, chuang2020debiased}.

    More recently, the behavior of CL has been analyzed from the perspective of \emph{alignment} and \emph{uniformity} properties of representations, demonstrating that these two properties are correlated with downstream performance~\citep{wang2020understanding}.
    We build on these results to make a connection to cross-entropy minimization from which we can derive identifiability results.%

\paragraph{Nonlinear ICA}
    Independent Components Analysis (ICA) attempts to find the underlying sources for multidimensional data. In the nonlinear case, said sources correspond to a well-defined nonlinear generative model $g$, which is assumed to be invertible (i.e., injective)~\citep{Hyvabook,Jutten10}. In other words, nonlinear ICA solves a demixing problem:
    Given observed data $\mathbf{x} = g(\mathbf{z})$, it aims to find a model $f$ that equals the inverse generative model $g^{-1}$, which allows for the original sources $\mathbf{z}$ to be recovered.
    
    \citet{hyvarinen2018nonlinear} show that the nonlinear demixing problem can be solved as long as the independent components are conditionally mutually independent with respect to some auxiliary variable. The authors further provide practical estimation methods for solving the nonlinear ICA problem~\citep{hyvarinen2016unsupervised,hyvarinen2017nonlinear}, similar in spirit to noise contrastive estimation (NCE; \citealp{Gutmann12JMLR}). Recent work has generalized this contribution to VAEs~\citep{khemakhem2020variational,locatello2020weakly,klindt2020slowvae}, as well as (invertible-by-construction) energy-based models~\citep{khemakhem2020ice}. We here extend this line of work to more general feed-forward networks trained using InfoNCE~\citep{oord2018representation}.
    
    In a similar vein, \citet{roeder2020linear} build on the work of \citet{hyvarinen2018nonlinear} to show that for a model family which includes InfoNCE, distribution matching implies parameter matching. In contrast, we associate the learned latent representation with the ground-truth generative factors, showing under what conditions the data generating process is inverted, and thus, the true latent factors are recovered.
    
\section{Theory}
    
    We will show a connection between contrastive learning and identifiability in the form of nonlinear ICA. For this, we introduce a feature encoder $f$ that maps observations $\xx$ to representations.
    We consider the widely used \emph{InfoNCE} loss, which often assumes $L^2$ normalized representations \citep{wu2018unsupervised, he2020momentum, tian2019contrastive, bachman2019learning,chen2020simple},
    \begin{align} \label{eq:contrastive_loss}
        &\lcontr(f; \tau, M) \quad := \\ 
        &\underset{\substack{
            (\x, \tx) \sim p_\mathsf{pos} \\
            \{\xx^-_i\}_{i=1}^M \overset{\text{i.i.d.}}{\sim} p_\mathsf{data}
        }}{\mathbb{E}} \left[\, {- \log \frac{e^{f(\xx)^{\mathsf{T}} f(\tx) / \tau }}{e^{f(\xx)^{\mathsf{T}} f(\tx) / \tau } + \sum\limits_{i=1}^M e^{f(\xx)^{\mathsf{T}} f(\xx^-_i) / \tau }}}\,\right]. \nonumber
    \end{align}
    Here $M\in\mathbb{Z}_+$ is a fixed number of negative samples, $p_{\rm{data}}$ is the distribution of all observations and $p_{\rm{pos}}$ is the distribution of positive pairs.
    This loss was motivated by the InfoMax principle \citep{linsker1988self}, and has been shown
    to be effective by many recent representation learning methods \citep{logeswaran2018efficient,wu2018unsupervised,tian2019contrastive,he2019momentum,hjelm2018learning,bachman2019learning,chen2020simple,baevski2020wav2vec}. Our theoretical results also hold for a loss function whose denominator only consists of the second summand across the negative samples (e.g., the SimCLR loss \citep{chen2020simple}). %
    
    In the spirit of existing literature on nonlinear ICA \cite{hyvarinen1999nonlinear, harmeling2003kernel,sprekeler2014extension,hyvarinen2016unsupervised,hyvarinen2017nonlinear, Gutmann12JMLR, hyvarinen2018nonlinear, khemakhem2020variational}, we assume that the observations $\xx \in \X$ are generated by an invertible (i.e., injective) generative process $g: \Z \to \X$, where $\X \subseteq \R^K$ is the space of observations and $\Z \subseteq \R^N$ with $N\leq K$ denotes the space of latent factors. Influenced by the commonly used feature normalization in InfoNCE, we further assume that $\Z$ is the unit hypersphere $\Sp^{N-1}$ (see Appx.~\ref{apx:gt_assumptions}).
    Additionally, we assume that the ground-truth marginal distribution of the latents of the generative process is uniform and that the conditional distribution (under which positive pairs have high density) is a von Mises-Fisher (vMF) distribution:
    \begin{align} \label{eq:vmf_conditional}
        p(\z) &= |\Z|^{-1}, \quad\quad p(\z|\tzz) = C_p^{-1} e^{\kappa \z^\top \tzz} \quad \text{with} \\ C_p :&= \int e^{\kappa \z^\top \tzz} \,\d\tzz = \text{const.}, \quad \xx = g(\z), \quad \txx = g(\tzz). \nonumber
    \end{align}
    
    Given these assumptions, we will show that if $f$ minimizes the contrastive loss $\lcontr$, then $f$ solves the demixing problem, i.e., inverts $g$ up to orthogonal linear transformations. %
    
    Our theoretical approach consists of three steps:
    (1) We demonstrate that $\lcontr$ can be interpreted as the cross-entropy between the (conditional) ground-truth and inferred latent distribution. %
    (2) Next, we show that encoders minimizing $\lcontr$ maintain distance, i.e., two latent vectors with distance $\alpha$ in the ground-truth generative model are mapped to points with the same distance $\alpha$ in the inferred representation. %
    (3) Finally, we leverage distance preservation to show that minimizers of $\lcontr$  invert the generative process up to orthogonal transformations.
    Detailed proofs are given in Appx.~\ref{apx:proofs}.
    
    Additionally, we will present similar results for general convex bodies in $\mathbb{R^N}$ and more general similarity measures, see Sec.~\ref{sec:extension_rn}. For this, the detailed proofs are given in Appx.~\ref{apx:rn_extension}. %
    
\subsection{Contrastive learning is related to cross-entropy minimization}
    From the perspective of nonlinear ICA, we are interested in understanding how the representations $f(\x)$ which minimize the contrastive loss $\lcontr$ (defined in Eq.~\eqref{eq:contrastive_loss}) are related to the ground-truth source signals $\z$. To study this relationship, we focus on the map $h = f\circ g$ between the recovered source signals $h(\z)$ and the true source signals $\z$. Note that this is merely for mathematical convenience; it does not necessitate knowledge regarding neither $g$ nor the ground-truth factors during learning (beyond the assumptions stated in the theorems).
    
    A core insight is a connection between the contrastive loss and the cross-entropy between the ground-truth latent distribution and a certain model distribution. For this, we expand the theoretical results obtained by \citet{wang2020understanding}:
    \vspace{\topsep}
    \begin{customtheorem}{\ref*{thm:extended_asym_inf_negatives_CE}}[$\lcontr$ converges to the cross-entropy between latent distributions] \label{thm:asym_inf_negatives_CE}
        If the ground-truth marginal distribution $p$ is uniform, then for fixed $\tau > 0$, as the number of negative samples $M \rightarrow \infty$, the (normalized) contrastive loss converges to
        \begin{equations}
            \lim_{M \rightarrow \infty} \lcontr(f; \tau, M) - \log M + \log |\Z| = \\ \expectunder{\z \sim p(\z)}{H(p(\cdot | \z), q_h(\cdot | \z))}
            \label{eq:contrastive_loss_CE_limit}
        \end{equations}
        where $H$ is the cross-entropy between the ground-truth conditional distribution $p$ over positive pairs and a conditional distribution $\qh$ parameterized by the model $f$,
        \begin{equations} \label{eq:qhjoint}
            \qh(\tzz|\z) &= C_h(\zz)^{-1} e^{h(\tzz)\T h(\zz) /\tau} \\ \text{with} \quad C_h(\zz) :&= \int e^{h(\tzz)\T h(\zz) /\tau} \,\d\tzz,
        \end{equations}
        where $C_h(\z)\in\R^{+}$ is the partition function of $\qh$ (see Appx.~\ref{apx:model_assumptions}).
    \end{customtheorem}
    
    Next, we show that the minimizers $h^{*}$ of the cross-entropy~(\ref{eq:qhjoint}) are isometries in the sense that $\kappa \z^\top\tzz = h^{*}(\z)^\top h^{*}(\tzz)$ for all $\z$ and $\tzz$. In other words, they preserve the dot product between $\z$ and $\tzz$. %
    \vspace{\topsep}
    \begin{customproposition}{\ref*{prop:extended_correct_model_ce_isometry}}[Minimizers of the cross-entropy maintain the dot product] \label{prop:correct_model_ce_isometry}
        Let $\Z = \Sp^{N-1}$, $\tau > 0$ and consider the ground-truth conditional distribution of the form $p(\tzz | \z) = C_p^{-1} \exp(\kappa \tzz^\top \zz)$. Let $h$ map onto a hypersphere with radius $\sqrt{\tau \kappa}$.\footnote{Note that in practice this can be implemented as a learnable rescaling operation as the last operation of the network $f$.} Consider the conditional distribution $q_h$ parameterized by the model, as defined above in Theorem~\ref{thm:asym_inf_negatives_CE}, where the hypothesis class for $h$ (and thus $f$) is assumed to be sufficiently flexible such that $p(\tzz | \zz)$ and $\qh(\tzz|\zz)$ can match.
        If $h$ is a minimizer of the cross-entropy $\E_{p(\tzz | \zz)}[- \log \qh(\tzz | \zz)]$, then $p(\tzz|\zz) = \qh(\tzz | \zz)$ and $\forall \z, \tzz: \kappa \z^\top\tzz = h(\z)^\top h(\tzz)$.
    \end{customproposition}

\subsection{Contrastive learning identifies ground-truth factors on the hypersphere}
    From the strong geometric property of isometry, we can now deduce a key property of the minimizers $h^*$: %
    \vspace{\topsep}
    \begin{customproposition}{\ref*{prop:extended_mazurulamspheres}}[Extension of the Mazur-Ulam theorem to hyperspheres and the dot product]
        \label{prop:mazurulamspheres}
        Let $\Z = \Sp^{N-1}$ and $\Z' = \Sp^{N-1}_{r}$ be the hyperspheres with radius $1$ and $r > 0$, respectively. If $h: \R^N \to \Z'$ is differentiable in the vicinity of $\Z$ and its restriction to $\Z$ maintains the dot product up to a constant factor, i.e., $\forall \z, \tzz \in \Z: r^2 \z^\top \tzz = h(\z)^\top h(\tzz)$, then $h$ is an orthogonal linear transformation scaled by $r$ for all $\zz \in \Z$.
    \end{customproposition}

    In the last step, we combine the previous propositions to derive our main result: the minimizers of the contrastive loss $\lcontr$ solve the demixing problem of nonlinear ICA up to linear transformations, i.e., they identify the original sources $\z$ for observations $g(\z)$ up to orthogonal linear transformations. For a hyperspherical space $\Z$ these correspond to combinations of permutations, rotations and sign flips.
    \vspace{\topsep}
    \begin{customtheorem}{\ref*{thm:extended_ident_matching}}\label{thm:ident_matching}
        Let $\Z = \Sp^{N-1}$, the ground-truth marginal be uniform, and the conditional a vMF distribution (cf. Eq.~\ref{eq:vmf_conditional}). Let the restriction of the mixing function $g$ to $\Z$ be injective and $h$ be differentiable in a vicinity of $\Z$. If the assumed form of $\qh$, as defined above, matches that of $p$, and if $f$ is differentiable and minimizes the CL loss as defined in Eq.~\eqref{eq:contrastive_loss}, then for fixed $\tau > 0$ and $M\to\infty$, $h = f \circ g$ is linear, i.e., $f$ recovers the latent sources up to an orthogonal linear transformation and a constant scaling factor.
    \end{customtheorem}
    Note that we do not assume knowledge of the ground-truth generative model $g$; we only make assumptions about the conditional and marginal distribution of the latents.
    On real data, it is unlikely that the assumed model distribution $\qh$ can exactly match the ground-truth conditional. We do, however, 
    provide empirical evidence that $h$ is still an affine transformation even if there is a severe mismatch, see Sec.~\ref{sec:experiments}.

\subsection{Contrastive learning identifies ground-truth factors on convex bodies in \texorpdfstring{$\mathbb{R}^N$}{RN}} \label{sec:extension_rn}
    While the previous theoretical results require $\Z$ to be a hypersphere, we will now show a similar theorem for the more general case of $\Z$ being a convex body in $\mathbb{R}^N$. Note that the hyperrectangle $[a_1, b_1] \times \ldots \times [a_N, b_N]$ is an example of such a convex body.
    
    We follow a similar three step proof strategy as for the hyperspherical case before:
    (1) We begin again by showing that a properly chosen contrastive loss on convex bodies corresponds to the cross-entropy between the ground-truth conditional and a distribution parametrized by the encoder. For this step, we additionally extend the results of \citet{wang2020understanding} to this latent space and loss function.
    (2) Next, we derive that minimizers of the loss function are isometries of the latent space. Importantly, we do not limit ourselves to a specific metric, thus the result is applicable to a family of contrastive objectives.
    (3) Finally, we show that these minimizers must be affine transformations.
    For a special family of conditional distributions (rotationally asymmetric generalized normal distributions~\citep{subbotin1923law}), we can further narrow the class of solutions to permutations and sign-flips. %
    For the detailed proofs, see Appx.~\ref{apx:rn_extension}. 
    
    As earlier, we assume that the ground-truth marginal distribution of the latents is uniform. However, we now assume that the conditional distribution is exponential:
    \begin{equations} \label{eq:rn_conditional}
        p(\z) &= |\Z|^{-1}, \quad\quad p(\z|\tzz) = C_p^{-1} e^{- \delta(\z, \tzz)} \quad \text{with} \\ C_p(\z) :&= \int e^{-\delta(\z, \tzz)} \,\d\tzz, \quad \xx = g(\z), \quad \txx = g(\tzz),
    \end{equations}
    where $\delta$ is a metric induced by a norm (see Appx.~\ref{apx:rn_gt_assumptions}).
    
    To reflect the differences between this conditional distribution and the one assumed for the hyperspherical case, we need to introduce an adjusted version of the contrastive loss in \eqref{eq:contrastive_loss}:  
    \begin{definition}[$\ldeltacontr$ objective] \label{def:delta_contrastive_loss}
        Let $\delta: \Z \times \Z \to \mathbb{R}$ be a metric on $\Z$. We define the general InfoNCE loss, which uses $\delta$ as a similarity measure, as
        
        \begin{align} \label{eq:delta_contrastive_loss}
            &\ldeltacontr(f; \tau, M) \quad :=\\
            &\underset{\substack{
                (\x, \tx) \sim p_\mathsf{pos} \\
                \{\xx^-_i\}_{i=1}^M \overset{\text{i.i.d.}}{\sim} p_\mathsf{data}
            }}{\mathbb{E}} \hspace{-1em}\Bigg[
            {- \log \frac{e^{-\delta(f(\xx), f(\tx)) / \tau }}{e^{\text{--}\delta(f(\xx), f(\tx)) / \tau } \hspace{-.3em}+\hspace{-.3em} \sum\limits_{i=1}^M e^{\text{--}\delta(f(\xx), f(\xx^\text{--}_i)) / \tau }}}\,\Bigg]. \nonumber
        \end{align}
    \end{definition}
    Note that this is a generalization of the InfoNCE criterion in Eq.~(\ref{eq:contrastive_loss}). In contrast to the objective above, the representations are no longer assumed to be $L^2$ normalized, and the dot-product is replaced with a more general similarity measure $\delta$.
    
    Analogous to the previously demonstrated case for the hypersphere, for convex bodies $\Z$, minimizers of the adjusted $\ldeltacontr$ objective solve the demixing problem of nonlinear ICA up to invertible linear transformations:
    \begin{customtheorem}{\ref*{thm:extended_rn_linear_identifiable}} \label{thm:rn_linear_identifiable}
        Let $\Z$ be a convex body in $\mathbb{R}^N$, $h = f\circ g:\Z\to\Z$, and $\delta$ be a metric or a semi-metric (cf. Lemma~\ref{lem:semimetric} in Appx.~\ref{apx:rn_ce_min_identifiability}), induced by a norm. Further, let the ground-truth marginal distribution be uniform and the conditional distribution be as Eq.~\eqref{eq:rn_conditional}. Let the mixing function $g$ be differentiable and injective. If the assumed form of $\qh$ matches that of $p$, i.e., 
        \begin{equations}
            \qh(\tzz|\z) &= C_q^{-1}(\zz)e^{-\delta(h(\tzz), h(\zz))/\tau}\quad \\ \text{with} \quad C_q(\zz) :&= \int e^{-\delta(h(\tzz), h(\zz))/\tau} \,\d\tzz,
        \end{equations}
        and if $f$ is differentiable and minimizes the $\ldeltacontr$ objective in Eq.~\eqref{eq:delta_contrastive_loss} for $M \to \infty$, we find that $h = f \circ g$ is invertible and affine, i.e., we recover the latent sources up to affine transformations.
    \end{customtheorem}
    Note that the model distribution $\qh$, which is implicitly described by the choice of the objective, must be of the same form as the ground-truth distribution $p$, i.e., both must be based on the same metric. Thus, identifying different ground-truth conditional distributions requires different contrastive $\ldeltacontr$ objectives.
    This result can be seen as a generalized version of Theorem~\ref{thm:ident_matching}, as it is valid for any convex body $\Z \subseteq \mathbb{R}^N$, allowing for a larger variety of conditional distributions.
    
    Finally, under the mild restriction that the ground-truth conditional distribution is based on an $L^p$ similarity measure for $p \geq1, p \neq 2$, $h$ identifies the ground-truth generative factors up to generalized permutations. A generalized permutation matrix $\A$ is a combination of a permutation and element-wise sign-flips, i.e., $\forall \z: (\A\z)_i = \alpha_i \z_{\sigma(i)}$ with $\alpha_i = \pm 1$ and $\sigma$ being a permutation.
    \begin{customtheorem}{\ref*{thm:extended_rn_permutation_identifiable}} \label{thm:rn_permutation_identifiable}
        Let $\Z$ be a convex body in $\mathbb{R}^N$, $h: \Z \to \Z$, and $\delta$ be an $L^\alpha$ metric or semi-metric (cf. Lemma~\ref{lem:semimetric} in Appx.~\ref{apx:rn_ce_min_identifiability}) for $\alpha \geq 1, \alpha \neq 2$. Further, let the ground-truth marginal distribution be uniform and the conditional distribution be as Eq.~\eqref{eq:rn_conditional}, and let the mixing function $g$ be differentiable and invertible. If the assumed form of $\qh(\cdot|\z)$ matches that of $p(\cdot|\z)$, i.e., both use the same metric $\delta$ up to a constant scaling factor, and if $f$ is differentiable and minimizes the $\ldeltacontr$ objective in Eq.~\eqref{eq:delta_contrastive_loss} for $M \to \infty$, we find that $h = f \circ g$ is a composition of input independent permutations, sign flips and rescaling.
    \end{customtheorem}

\section{Experiments} \label{sec:experiments}

\subsection{Validation of theoretical claim} \label{sec:toy_experiments}
    We validate our theoretical claims under both perfectly matching and violated conditions regarding the ground-truth marginal and conditional distributions. We consider source signals of dimensionality $N=10$, and sample pairs of source signals in two steps: First, we sample from the marginal $p(\z)$. For this, we consider both uniform distributions which match our assumptions and non-uniform distributions (e.g., a normal distribution) which violate them. Second, we generate the positive pair by sampling from a conditional distribution $p(\tzz | \z)$.
    Here, we consider matches with our assumptions on the conditional distribution (von Mises-Fisher for $\Z = \Sp^{N-1}$) as well as violations (e.g. normal, Laplace or generalized normal distribution for $\Z = \Sp^{N-1}$). Further, we consider spaces beyond the hypersphere, such as the bounded box (which is a convex body) and the unbounded $\R^N$.
    
    We generate the observations with a multi-layer perceptron (MLP), following previous work~\citep{hyvarinen2016unsupervised,hyvarinen2017nonlinear}.
    Specifically, we use three hidden layers with leaky ReLU units and random weights; to ensure that the MLP $g$ is invertible, we control the condition number of the weight matrices.
    For our feature encoder $f$, we also use an MLP with leaky ReLU units, where the assumed space is denoted by the normalization, or lack thereof, of the encoding. Namely, for the hypersphere (denoted as \emph{Sphere}) and the hyperrectangle (denoted as \emph{Box}) we apply an $L^2$ and $L^\infty$ normalization, respectively. For flexibility in practice, we parameterize the normalization magnitude of the \emph{Box}, including it as part of the encoder's learnable parameters. On the hypersphere we optimize $\lcontr$ and on the hyperrectangle as well as the unbounded space we optimize $\ldeltacontr$. For further details, see Appx.~\ref{apx:experiment_details}.  %
    
    To test for identifiability up to affine transformations, we fit a linear regression between the ground-truth and recovered sources and report the coefficient of determination ($R^2$). To test for identifiability up to generalized permutations, we leverage the mean correlation coefficient (MCC), as used in previous work~\citep{hyvarinen2016unsupervised,hyvarinen2017nonlinear}. For further details, see Appx.~\ref{apx:experiment_details}.
    
    \begin{table*}[t]
        \centering
        \caption{Identifiability up to affine transformations. Mean $\pm$ standard deviation over $5$ random seeds. Note that only the first row corresponds to a setting that matches (\cmark) our theoretical assumptions, while the others show results for violated assumptions (\xmark; see column \emph{M.}). Note that the identity score only depends on the ground-truth space and the marginal distribution defined for the generative process, while the supervised score additionally depends on the space assumed by the model. 
        }
        \resizebox{\textwidth}{!}{%
        \begin{tabular}{ccccccccc}
            \toprule
            \multicolumn{3}{c}{Generative process $g$} & \multicolumn{3}{c}{Model $f$} & \multicolumn{3}{c}{$R^2$ Score [\%]} \\
            Space & $p(\cdot)$ & $p(\cdot|\cdot)$ & Space & $\qh(\cdot|\cdot)$ & M. & Identity & Supervised & Unsupervised \\
            \midrule
            Sphere & Uniform & vMF($\kappa{=}1$) & Sphere & vMF($\kappa{=}1$) & \cmark & $66.98 \pm 2.79$ & $99.71  \pm 0.05$ & $99.42 \pm 0.05$ \\
            Sphere & Uniform & vMF($\kappa{=}10$) & Sphere & vMF($\kappa{=}1$) & \xmark & \dittotikz & \dittotikz & $99.86 \pm 0.01$ \\
            Sphere & Uniform & Laplace($\lambda{=}0.05$) & Sphere & vMF($\kappa{=}1$) & \xmark & \dittotikz & \dittotikz & $99.91 \pm 0.01$ \\
            Sphere & Uniform & Normal($\sigma{=}0.05$) & Sphere & vMF($\kappa{=}1$) & \xmark & \dittotikz & \dittotikz & $99.86 \pm 0.00$\\
            \midrule
            Box & Uniform & Normal($\sigma{=}0.05$) & Unbounded & Normal & \xmark & $67.93 \pm 7.40$ & $99.78 \pm 0.06$ & $99.60 \pm 0.02$ \\
            Box & Uniform & Laplace($\lambda{=}0.05$) & Unbounded & Normal & \xmark & \dittotikz & \dittotikz & $99.64 \pm 0.02$ \\
            Box & Uniform & Laplace($\lambda{=}0.05$) & Unbounded & GenNorm($\beta{=}3$) & \xmark & \dittotikz & \dittotikz & $99.70 \pm 0.02$\\
            Box & Uniform & Normal($\sigma{=}0.05$) & Unbounded & GenNorm($\beta{=}3$) & \xmark & \dittotikz & \dittotikz & $99.69 \pm 0.02$\\
            \midrule
            Sphere & Normal($\sigma{=}1$) & Laplace($\lambda{=}0.05$) & Sphere & vMF($\kappa{=}1$) & \xmark & $63.37 \pm 2.41$ & $99.70 \pm 0.07$ & $99.02 \pm 0.01$ \\
            Sphere & Normal($\sigma{=}1$) & Normal($\sigma{=}0.05$) & Sphere & vMF($\kappa{=}1$) & \xmark & \dittotikz & \dittotikz & $99.02 \pm 0.02$ \\
            \midrule          Unbounded & Laplace($\lambda{=}1$) & Normal($\sigma{=}1$) & Unbounded & Normal & \xmark & $62.49 \pm 1.65$ & $99.65 \pm 0.04$ & $98.13 \pm 0.14$ \\ Unbounded & Normal($\sigma{=}1$) & Normal($\sigma{=}1$) & Unbounded & Normal & \xmark & $63.57 \pm 2.30$ & $99.61 \pm 0.17$ & $98.76 \pm 0.03$ \\
            \bottomrule
        \end{tabular}}
        \label{tab:results_linear}
    \end{table*}
    
    \begin{table*}[t]
        \centering
        \caption{Identifiability up to generalized permutations, averaged over $5$ runs. 
        Note that while Theorem~\ref{thm:extended_rn_permutation_identifiable} requires the model latent space to be a convex body and $p(\cdot|\cdot)=\qh(\cdot|\cdot)$, we find that empirically either is sufficient.
        The results are grouped in four blocks corresponding to different types and degrees of violation of assumptions of our theory showing identifiability up to permutations: (1) no violation, violation of the assumptions on either the (2) space or (3) the conditional distribution, or (4) both.
        }
        \resizebox{\textwidth}{!}{%
        \begin{tabular}{ccccccccc}
            \toprule \multicolumn{3}{c}{Generative process $g$} & \multicolumn{3}{c}{Model $f$} & \multicolumn{3}{c}{MCC Score [\%]} \\
            Space & $p(\cdot)$ & $p(\cdot|\cdot)$ & Space & $\qh(\cdot|\cdot)$ & M. & Identity & Supervised & Unsupervised \\
            \midrule
            Box & Uniform & Laplace($\lambda{=}0.05$) & Box & Laplace & \cmark & $46.55 \pm 1.34$ & $99.93 \pm 0.03$ & $98.62 \pm 0.05$ \\
            Box & Uniform & GenNorm($\beta{=}3$; $\lambda{=}0.05$) & Box & GenNorm($\beta{=}3$) & \cmark & \dittotikz & \dittotikz & $99.90 \pm 0.06$ \\
            \midrule
            Box & Uniform & Normal($\sigma{=}0.05$) & Box & Normal & \xmark & \dittotikz & \dittotikz & $99.77 \pm 0.01$ \\
            Box & Uniform & Laplace($\lambda{=}0.05$) & Box & Normal & \xmark & \dittotikz & \dittotikz & $99.76 \pm 0.02$ \\
            Box & Uniform & GenNorm($\beta{=}3$; $\lambda{=}0.05$) & Box & Laplace & \xmark & \dittotikz & \dittotikz & $98.80 \pm 0.02$ \\
            \midrule
            Box & Uniform & Laplace($\lambda{=}0.05$) & Unbounded & Laplace & \xmark & \dittotikz & $99.97 \pm 0.03$ & $98.57 \pm 0.02$ \\
            Box & Uniform & GenNorm($\beta{=}3$; $\lambda{=}0.05$) & Unbounded & GenNorm($\beta{=}3$) & \xmark & \dittotikz & \dittotikz & $99.85 \pm 0.01$ \\
            \midrule
            Box & Uniform & Normal($\sigma{=}0.05$) & Unbounded & Normal & \xmark & \dittotikz & \dittotikz & $58.26 \pm 3.00$ \\
            Box & Uniform & Laplace($\lambda{=}0.05$) & Unbounded & Normal & \xmark & \dittotikz & \dittotikz & $59.67 \pm 2.33$ \\
            Box & Uniform & Normal($\sigma{=}0.05$) & Unbounded & GenNorm($\beta{=}3$) & \xmark & \dittotikz & \dittotikz & $43.80 \pm 2.15$ \\
            \bottomrule
        \end{tabular}}
        \label{tab:perm_results}
    \end{table*}
    
    We evaluate both identifiability metrics for three different model types.
    First, we ensure that the problem requires nonlinear demixing by considering the identity function for model $f$, which amounts to scoring the observations against the sources (\textbf{Identity Model}).
    Second, we ensure that the problem is solvable within our model class by training our model $f$ with supervision, minimizing the mean-squared error between $f(g(\z))$ and $\z$ (\textbf{Supervised Model}). Third, we fit our model without supervision using a contrastive loss (\textbf{Unsupervised Model}).
      
    Tables~\ref{tab:results_linear} and~\ref{tab:perm_results} show results evaluating identifiability up to affine transformations and generalized permutations, respectively. When assumptions match (see column M.), CL recovers a score close to the empirical upper bound.
    Mismatches in assumptions on the marginal and conditional do not lead to a significant drop in performance with respect to affine identifiability, but do for permutation identifiability compared to the empirical upper bound.
    In many practical scenarios, we use the learned representations to solve a downstream task, thus, identifiability up to affine transformations is often sufficient.
    However, for applications where identification of the individual generative factors is desirable, some knowledge of the underlying generative process is required to choose an appropriate loss function and feature normalization.
    Interestingly, we find that for convex bodies, we obtain identifiability up to permutation even in the case of a normal conditional, which likely is due to the axis-aligned box geometry of the latent domain.
    Finally, note that the drop in performance for identifiability up to permutations in the last group of Tab.~\ref{tab:perm_results} is a natural consequence of either the ground-truth or the assumed conditional being rotationally symmetric, e.g., a normal distribution, in an unbounded space. Here, rotated versions of the latent space are indistinguishable and, thus, the model cannot align the axes of the reconstruction with that of the ground-truth latent space, resulting in a lower score.
    
    To zoom in on how violations of the uniform marginal assumption influence the identifiability achieved by a model in practice, we perform an ablation on the marginal distribution by interpolating between the theoretically assumed uniform distribution and highly locally concentrated distributions.
    In particular, we consider two cases: (1) a sphere ($\mathcal{S}^9$) with a vMF marginal around its north pole for different concentration parameters $\kappa$; (2) a box ($[0,1]^{10}$) with a normal marginal around the box's center for different standard deviations $\sigma$.
    For both cases, Fig.~\ref{fig:uniformity_violation} shows the $R^2$ score as a function of the concentration $\kappa$ and $1/\sigma^2$ respectively (black). As a reference, the concentration of the used conditional distribution is highlighted as a dashed line.
    In addition, we also display the probability mass (0--100\%) that needs to be moved for converting the used marginal distribution (i.e., vMF or normal) into the assumed uniform marginal distribution (blue) as an intuitive measure of the mismatch (i.e., $\frac{1}{2}\int |p(\mathbf{z})\mathrm{-}p_{\mathrm{uni}}|\, \mathrm{d}\mathbf{z}$).
    While, we observe significant robustness to mismatch, in both cases, we see performance drop drastically once the marginal distribution is more concentrated than the conditional distribution of positive pairs. In such scenarios, positive pairs are indistinguishable from negative pairs. 

    \begin{figure}
        \centering
        \includegraphics[width=.9\linewidth]{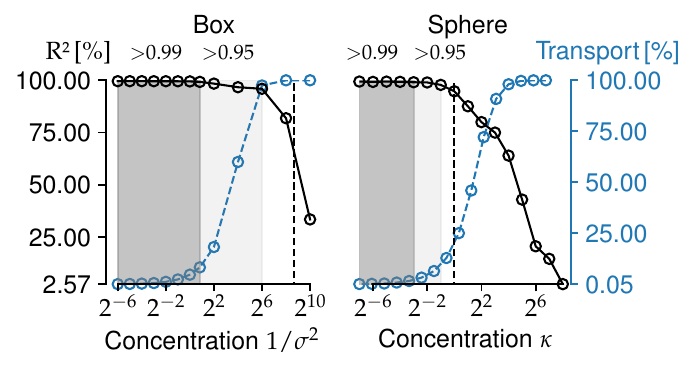}
        \vspace*{-0.5cm}
        \caption{Varying degrees of violation of the uniformity assumption for the marginal distribution. The figure shows the $R^2$ score measuring identifiability up to linear transformations (black) as well as the difference between the used marginal and assumed uniform distribution in terms of probability mass (blue) as a function of the marginal's concentration. The black dotted line indicates the concentration of the used conditional distribution.}
        \label{fig:uniformity_violation}
    \end{figure}
    
\subsection{Extensions to image data}
    Previous studies have demonstrated that representation learning using constrastive learning scales well to complex natural image data \citep{chen2020simple, chen2020big, henaff2020data}.
    Unfortunately, the true generative factors of natural images are inaccessible, thus we cannot evaluate identifiability scores.

    We consider two alternatives.
    First, we evaluate on the recently proposed benchmark \textit{KITTI Masks}~\citep{klindt2020slowvae}, which is composed of segmentation masks of natural videos.
    Second, we contribute a novel benchmark (\textit{3DIdent}; cf. Fig.~\ref{fig:3dident_examples}) which features aspects of natural scenes, e.g. a complex 3D object and different lighting conditions, while still providing access to the continuous ground-truth factors. For further details, see Appx.~\ref{apx:3dident_comparison}. \textit{3DIdent} is available at \href{https://zenodo.org/record/4502485/}{zenodo.org/record/4502485}.

\subsubsection{KITTI Masks} \label{sec:kitti_experiments}
    KITTI Masks~\citep{klindt2020slowvae} is composed of pedestrian segmentation masks extracted from an autonomous driving vision benchmark KITTI-MOTS~\citep{geiger2012are}, with natural shapes and continuous natural transitions. We compare to SlowVAE~\citep{klindt2020slowvae}, the state-of-the-art on the considered dataset. In our experiments, we use the same training hyperparameters (for details see Appx.~\ref{apx:experiment_details}) and (encoder) architecture as \citet{klindt2020slowvae}. The positive pairs consist of nearby frames with a time separation $\overline{\Delta t}$.
    
    \begin{table}
        \centering
        \caption{\textbf{KITTI Masks}. Mean $\pm$ standard deviation over 10 random seeds. $\overline{\Delta t}$ indicates the average temporal distance of frames used.}
        \label{table:MOTSComp}
        \small
        \begin{tabular}{clll}
            \toprule
                    & Model & Model Space & MCC [\%] \\
            \midrule
            \parbox[t]{16mm}{\multirow{5}{*}{$\overline{\Delta t}=0.05s$}} & SlowVAE &  Unbounded & 66.1 $\pm$ 4.5 \\
            & Laplace &  Unbounded & 77.1 $\pm$ 1.0 \\
            &  Laplace &  Box & 74.1 $\pm$ 4.4 \\
            &  Normal &  Unbounded & 58.3 $\pm$ 5.4 \\
            &  Normal &  Box & 59.9 $\pm$ 5.5 \\
             \midrule
           \parbox[t]{16mm}{\multirow{5}{*}{$\overline{\Delta t}=0.15s$}} &  SlowVAE & Unbounded & 79.6 $\pm$ 5.8 \\
            &  Laplace &  Unbounded & 79.4 $\pm$ 1.9 \\
            &  Laplace &  Box & 80.9 $\pm$ 3.8 \\
            &  Normal &  Unbounded & 60.2 $\pm$ 8.7 \\
            &  Normal &  Box & 68.4 $\pm$ 6.7 \\
            \bottomrule
        \end{tabular}
    \end{table}
    
    As argued and shown in \citet{klindt2020slowvae}, the transitions in the ground-truth latents between nearby frames is sparse. Unsurprisingly then, Table~\ref{table:MOTSComp} shows that assuming a Laplace conditional as opposed to a normal conditional in the contrastive loss leads to better identification of the underlying factors of variation. %
    SlowVAE also assumes a Laplace conditional~\citep{klindt2020slowvae} but appears to struggle if the frames of a positive pair are too similar ($\overline{\Delta t}=0.05s$).
    This degradation in performance is likely due to the limited expressiveness of the decoder deployed in SlowVAE.

\subsubsection{3DIdent} \label{sec:3dident_experiments}
    
\paragraph{Dataset description}
    \begin{figure*}[htb]
        \centering
        \includegraphics[width=\textwidth]{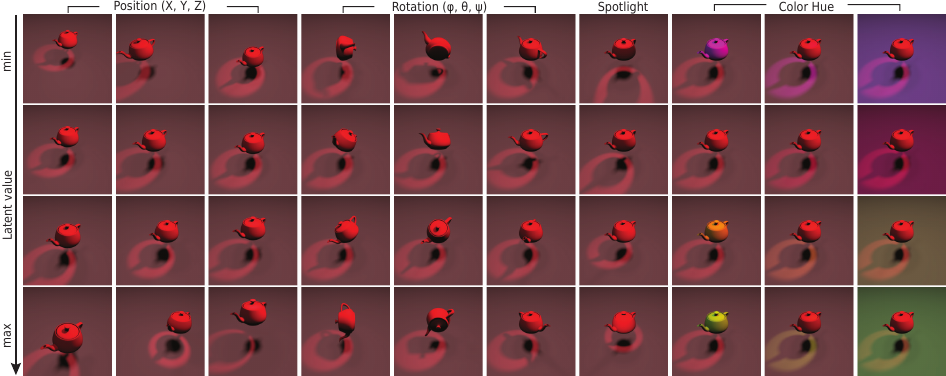}
        \caption{\textbf{3DIdent}. Influence of the latent factors $\z$ on the renderings $\x$. Each column corresponds to a traversal in one of the ten latent dimensions while the other dimensions are kept fixed. %
        }
        \label{fig:3dident_examples}
    \end{figure*}
    
    We build on \citep{johnson2017clevr} and use the Blender rendering engine \citep{blender} to create visually complex 3D images (see Fig.~\ref{fig:3dident_examples}). Each image in the dataset shows a colored 3D object which is located and rotated above a colored ground in a 3D space. Additionally, each scene contains a colored spotlight focused on the object and located on a half-circle around the scene. The observations are encoded with an RGB color space, and the spatial resolution is $224\times224$ pixels.

    The images are rendered based on a $10$-dimensional latent, where: (1) three dimensions describe the XYZ position, (2) three dimensions describe the rotation of the object in Euler angles, (3) two dimensions describe the color of the object and the ground of the scene, respectively, and (4) two dimensions describe the position and color of the spotlight. We use the HSV color space to describe the color of the object and the ground with only one latent each by having the latent factor control the hue value. For more details on the dataset see Sec.~\ref{apx:3dident_details}.
    
    The dataset contains $250\,000$ observation-latent pairs where the latents are uniformly sampled from the hyperrectangle $\Z$. To sample positive pairs $(\z, \tzz)$ we first sample a value $\tzz'$ from the data conditional $p(\tzz'|\z)$, and then use nearest-neighbor matching\footnote{We used an Inverted File Index (IVF) with Hierarchical Navigable Small World (HNSW) graph exploration for fast indexing.} implemented by FAISS \citep{JDH17} to find the latent $\tzz$ closest to $\tzz'$ (in $L^2$ distance) for which there exists an image rendering. In addition, unlike previous work~\citep{locatello2018challenging}, we create a hold-out test set with $25\,000$ distinct observation-latent pairs. 

\paragraph{Experiments and Results}
    \begin{table*}[htb]
        \vspace{-0.05cm}
        \centering
        \caption{Identifiability up to affine transformations on the test set of 3DIdent. Mean $\pm$ standard deviation over 3 random seeds. As earlier, only the first row corresponds to a setting that matches the theoretical assumptions for linear identifiability; the others show distinct violations. Supervised training with unbounded space achieves scores of $R^2=(98.67 \pm 0.03)$\% and $\text{MCC}=(99.33 \pm 0.01)$\%. The last row refers to using the image augmentations suggested by \citet{chen2020simple} to generate positive image pairs.
        For performance on the training set, see Appx.~Table~\ref{tab:3dident_results_train}.
        }
        \small
        \begin{tabular}{ccccccc}
            \toprule
            Dataset & \multicolumn{3}{c}{Model $f$} & Identity [\%] & \multicolumn{2}{c}{Unsupervised [\%]} \\
            $p(\cdot|\cdot)$ & Space & $\qh(\cdot|\cdot)$ & M. & $R^2$ & $R^2$ & MCC \\
            \midrule
            
            Normal & Box & Normal & \cmark & $5.25 \pm 1.20$ & $96.73 \pm 0.10$ & $98.31 \pm 0.04$\\            

            Normal & Unbounded & Normal & \xmark & \dittotikz & $96.43 \pm 0.03$ & $54.94 \pm 0.02$\\

            Laplace & Box & Normal & \xmark & \dittotikz & $96.87 \pm 0.08$ & $98.38 \pm 0.03$\\
            
            Normal & Sphere & vMF & \xmark & \dittotikz & $65.74 \pm 0.01$ & $42.44 \pm 3.27$\\
            
            Augm. & Sphere & vMF & \xmark & \dittotikz & $45.51 \pm 1.43$ & $46.34 \pm 1.59$\\

            \bottomrule
        \end{tabular}
        \label{tab:3dident_results_test}
    \end{table*}
    
    We train a convolutional feature encoder $f$ composed of a ResNet18 architecture~\citep{he2015deep} and an additional fully-connected layer, with a LeakyReLU nonlinearity as the hidden activation. For more details, see Appx.~\ref{apx:experiment_details}. Following the same methodology as in Sec.~\ref{sec:toy_experiments}, i) depending on the assumed space, the output of the feature encoder is normalized accordingly and ii) in addition to the CL models, we also train a supervised model to serve as an upper bound on performance. We consider normal and Laplace distributions for positive pairs. Note, that due to the finite dataset size we only sample from an approximation of these distributions.
    
    As in Tables~\ref{tab:results_linear} and~\ref{tab:perm_results}, the results in Table~\ref{tab:3dident_results_test} demonstrate that CL reaches scores close to the topline (supervised) performance, and mismatches between the assumed and ground-truth conditional distribution do not harm the performance significantly. However, if the hypothesis class of the encoder is too restrictive to model the ground-truth conditional distribution, we observe a clear drop in performance, i.e., mapping a box onto a sphere. Note, that this corresponds to the InfoNCE objective for $L^2$-normalized representations, commonly used for self-supervised representation learning~\citep{wu2018unsupervised, he2020momentum, tian2019contrastive, bachman2019learning,chen2020simple}.
    Finally, the last result shows that leveraging image augmentations~\citep{chen2020simple} as opposed to sampling from a specified conditional distribution of positive pairs $p(\cdot|\cdot)$ results in a performance drop. For details on the experiment, see Appx. Sec.~\ref{apx:experiment_details}.  We explain this with the greater mismatch between the conditional distribution assumed by the model and the conditional distribution induced by the augmentations. 
    In all, we demonstrate validation of our theoretical claims even for generative processes with higher visual complexity than those considered in Sec.~ \ref{sec:toy_experiments}.

\section{Conclusion}
    We showed that objectives belonging to the InfoNCE family, the basis for a number of state-of-the-art techniques in self-supervised representation learning, can uncover the true generative factors of variation underlying the observational data. To succeed, these objectives implicitly encode a few weak assumptions about the statistical nature of the underlying generative factors. While these assumptions will likely not be exactly matched in practice, we showed empirically that the underlying factors of variation are identified even if theoretical assumptions are severely violated.
    
    Our theoretical and empirical results suggest that the representations found with contrastive learning implicitly (and approximately) invert the generative process of the data. This could explain why the learned representations are so useful in many downstream tasks. It is known that a decisive aspect of contrastive learning is the right choice of augmentations that form a positive pair. We hope that our framework might prove useful for clarifying the ways in which certain augmentations affect the learned representations, and for finding improved augmentation schemes.
 
    Furthermore, our work opens avenues for constructing more effective contrastive losses. As we demonstrate, imposing a contrastive loss informed by characteristics of the latent space can considerably facilitate inferring the correct semantic descriptors, and thus boost performance in downstream tasks. 
    While our framework already allows for a variety of \emph{conditional} distributions, it is an interesting open question how to adapt it to \emph{marginal} distributions beyond the uniform implicitly encoded in InfoNCE. Also, future work may extend our theoretical framework by incorporating additional assumptions about our visual world, such as compositionality, hierarchy or objectness. 
    Accounting for such inductive biases holds enormous promise in forming the basis for the next generation of self-supervised learning algorithms. 
    
    Taken together, we lay a strong theoretical foundation for not only understanding but extending the success of state-of-the-art self-supervised learning techniques.

\subsection*{Author contributions}
    The project was initiated by WB.
    RSZ, StS and WB jointly derived the theory. RSZ and YS implemented and executed the experiments. The 3DIdent dataset was created by RSZ with feedback from StS, YS, WB and MB.
    RSZ, YS, StS and WB contributed to the final version of the manuscript.

\subsection*{Acknowledgements}
    We thank
    Muhammad Waleed Gondal,
    Ivan Ustyuzhaninov,
    David Klindt,
    Lukas Schott,
    Luisa Eck,
    and Kartik Ahuja
    for helpful discussions.
    We thank Bozidar Antic, Shubham Krishna and Jugoslav Stojcheski for ideas regarding the design of 3DIdent.
    We thank the International Max Planck Research School for Intelligent Systems (IMPRS-IS) for supporting RSZ, YS and StS.
    StS acknowledges his membership in the European Laboratory for Learning and Intelligent Systems (ELLIS) PhD program.
    We acknowledge support from the German Federal Ministry of Education and Research (BMBF) through the Competence Center for Machine Learning (TUE.AI, FKZ 01IS18039A) and the Bernstein Computational Neuroscience Program T\"ubingen (FKZ: 01GQ1002). WB acknowledges support via his Emmy Noether Research Group funded by the German Science Foundation (DFG) under grant no. BR 6382/1-1 as well as support by Open Philantropy and the Good Ventures Foundation. MB and WB acknowledge funding from the MICrONS program of the Intelligence Advanced Research Projects Activity (IARPA) via Department of Interior/Interior Business Center (DoI/IBC) contract number D16PC00003.

\clearpage
\bibliography{references}
\bibliographystyle{icml2020}

\clearpage
\appendix

\section{Appendix}
\subsection{Extended Theory for Hyperspheres}
\subsubsection{Assumptions} \label{apx:assumptions}
\paragraph{Generative Process} \label{apx:gt_assumptions}
    Let the generator $g : \R^N \to \X$ with $\X \subseteq \R^K$ and $K\geq N$. Further, let the restriction of $g$ to the space $\Z=\Sp^{N-1} \subset \R^N$ be injective and $g$ be differentiable in the vicinity of $\Z$.
    We assume that the marginal distribution $p(\z)$ over latent variables $\z \in \Z$ is uniform:
    \begin{align}
        p(\z) = \frac{1}{|\Z|}.
    \end{align}
    Further, we assume that the conditional distribution over positive pairs $p(\tzz | \z)$ is a von Mises-Fisher (vMF) distribution
    \begin{align}
        p(\tzz|\z) &= C_p^{-1} e^{\kappa \z^\top \tzz}\quad \\ \text{with} \quad C_p :&= \int e^{\kappa \boldsymbol{\eta}^\top\tzz} \,\d\tzz ,
    \end{align}
    where $\kappa$ is a parameter controlling the width of the distribution and  $\boldsymbol{\eta}$ is any vector on the hypersphere.
    Finally, we assume that during training one has access to observations $\x$, which are samples from these distributions transformed by the generator function $g$.

\paragraph{Model} \label{apx:model_assumptions}
    Let $f : \X \to \Sp_r^{N-1}$, where $\Sp_r^{N-1}$ denotes a hypersphere with radius $r$. The parameters of this model are optimized using contrastive learning. We associate a conditional distribution $\qh(\tzz|\zz)$ with our model $f$ through $h = f \circ g$ and
    \begin{equations}
        \qh(\tzz|\z) &= C_q^{-1}(\zz)e^{h(\tzz)^\top h(\zz)/\tau}\quad \\ \text{with} \quad C_q(\zz) :&= \int e^{h(\tzz)^\top h(\zz)/\tau} \,\d\tzz,
    \end{equations}
    where $C_q(\zz)$ is the partition function and $\tau > 0$ is a scale parameter.

\subsubsection{Proofs for Sec.~3} \label{apx:proofs}
    We begin by recalling a result of \citet{wang2020understanding}, where the authors show an asymptotic relation between the contrastive loss $\lcontr$ and two loss functions, the \emph{alignment} loss $\lalign$ and the \emph{uniformity} loss $\lunif$:
    \begin{customproposition}{A}[Asymptotics of $\lcontr$, \citealp{wang2020understanding}]\label{prop:asym_inf_negatives}
        For fixed $\tau > 0$, as the number of negative samples $M \rightarrow \infty$, the (normalized) contrastive loss converges to
        \begin{equations}
            \lim_{M \rightarrow \infty} \lcontr(f; \tau, M) - \log M \\ 
            = \lalign(f; \tau) + \lunif(f; \tau),
            \label{eq:contrastive_loss_limit}
        \end{equations}
        where
        \begin{equations}
            \lalign(f; \tau) &:= - \frac{1}{\tau} \expectunder{(\tzz, \zz) \sim p(\tzz, \zz)}{(f \circ g)(\zz)\T (f \circ g)(\zz)} \\
            \lunif(f; \tau) &:= \expectunder{\zz \sim p(\zz)}{\log \expectunder{\tzz \sim p(\tzz)}{e^{(f \circ g)(\tzz)\T (f \circ g)(\zz)/\tau}}}.
        \end{equations}
    \end{customproposition}
    \begin{proof}
        See Theorem~1 of \citet{wang2020understanding}. Note that they originally formulated the losses in terms of observations $\x$ and not in terms of the latent variables $\z$. However, this modified version simplifies notation in the following.
    \end{proof}

    Based on this result, we show that the contrastive loss $\lcontr$ asymptotically converges to the cross-entropy between the ground-truth conditional $p$ and our assumed model conditional distribution $q_h$, up to a constant.
    This is notable, because given the correct model specification for $q_h$, it is well-known that the cross-entropy is minimized iff $q_h = p$, i.e., the ground-truth conditional distribution and the model distribution will match.
    
    \begin{theorem}[$\lcontr$ converges to the cross-entropy between latent distributions] \label{thm:extended_asym_inf_negatives_CE}
        If the ground-truth marginal distribution $p$ is uniform, then for fixed $\tau > 0$, as the number of negative samples $M \rightarrow \infty$, the (normalized) contrastive loss converges to
        \begin{equations}
            \lim_{M \rightarrow \infty} \lcontr(f; \tau, M) - \log M + \log |\Z| = \\ \expectunder{\z \sim p(\z)}{H(p(\cdot | \z), q_h(\cdot | \z))}
            \label{eq:extended_contrastive_loss_CE_limit}
        \end{equations}
        where $H$ is the cross-entropy between the ground-truth conditional distribution $p$ over positive pairs and a conditional distribution $\qh$ parameterized by the model $f$, and $C_h(\z)\in\R^{+}$ is the partition function of $\qh$ (see Appendix~\ref{apx:model_assumptions}):
        \begin{equations} \label{eq:extended_qhjoint}
            \qh(\tzz|\z) &= C_h(\zz)^{-1} e^{h(\tzz)\T h(\zz) /\tau}  \\ \text{with}\quad C_h(\zz) :&= \int e^{h(\tzz)\T h(\zz) /\tau} \,\d\tzz.
        \end{equations}
    \end{theorem}
    \begin{proof}
        The cross-entropy between the conditional distributions $p$ and $q_h$ is given by
        \begin{align}
            &\expectunder{\z \sim p(\z)}{H(p(\cdot | \z), q_h(\cdot | \z))} \\
            =& \expectunder{\z \sim p(\z)}{\expectunder{
            \tzz \sim p(\tzz | \z)}{-\log q_h(\tzz| \z)}} \\
            =& \expectunder{
            \tzz,\zz\sim p(\tzz , \z)}{- \frac{1}{\tau}h(\tzz)^\top h(\zz) + \log C_h(\z)} \\
            =& - \frac{1}{\tau}\expectunder{
            \tzz,\zz\sim p(\tzz , \zz)}{h(\tzz)^\top h(\zz)} + \expectunder{\zz \sim p(\zz)}{\log C_h(\zz)}. \\
            \intertext{Using the definition of $C_h$ in Eq.~\eqref{eq:extended_qhjoint} we obtain}
            =& - \frac{1}{\tau} \expectunder{
            \tzz,\zz\sim p(\tzz , \zz)}{h(\tzz)^\top h(\zz)} \\
            &\quad\quad +\expectunder{\zz \sim p(\zz)}{\log \int_\Z  e^{h(\tzz)^\top h(\zz)/\tau} \,\d\tzz}.\\
            \intertext{By assumption the marginal distribution is uniform, i.e., $p(\z) = |\Z|^{-1}$. We expand by $|\Z| |\Z|^{-1}$ and estimate the integral by sampling from $p(\z) = |\Z|^{-1}$, yielding}
            =& - \frac{1}{\tau}\expectunder{
            \tzz,\zz\sim p(\tzz , \zz)}{h(\tzz)^\top h(\zz)} \\
            &\quad\quad +\expectunder{\zz \sim p(\zz)}{\log |\Z| \expectunder{\tzz \sim p(\tzz)}{e^{h(\tzz)^\top h(\zz)/\tau}}}\\
            =& - \frac{1}{\tau}\expectunder{
            \tzz,\zz\sim p(\tzz , \zz)}{h(\tzz)^\top h(\zz)}  \\
            &\quad\quad +\expectunder{\zz \sim p(\zz)}{\log \expectunder{\tzz \sim p(\tzz)}{e^{h(\tzz)^\top h(\zz)/\tau}}} + \log |\Z|. \\
            \intertext{By inserting the definition $h = f \circ g$,}
            =& - \frac{1}{\tau}\expectunder{
            \tzz,\zz\sim p(\tzz , \zz)}{(f \circ g)(\tzz)^\top (f \circ g)(\zz)} \\
            &\quad\quad +\expectunder{\zz \sim p(\zz)}{\log \expectunder{\tzz \sim p(\tzz)}{e^{(f \circ g)(\tzz)^\top (f \circ g)(\zz) /\tau}}} \\
            &\quad\quad + \log |\Z|, \\
            \intertext{we can identify the losses introduced in Proposition~\ref{prop:asym_inf_negatives},}
            =& \lalign(f; \tau) + \lunif(f; \tau)  + \log |\Z| ,
            \intertext{which recovers the original alignment term and the uniformity term for maximimizing entropy by means of a von Mises-Fisher KDE up to the constant $\log |\Z|$. According to Proposition~\ref{prop:asym_inf_negatives} this equals}
            =& \lim_{M \rightarrow \infty} \lcontr(f; \tau, M) - \log M + \log |Z|,
        \end{align}
        which concludes the proof.
    \end{proof}

    \vspace{\topsep}
    \begin{proposition}[Minimizers of the cross-entropy maintain the dot product] \label{prop:extended_correct_model_ce_isometry}
        Let $\Z = \Sp^{N-1}$, $\tau > 0$ and consider the ground-truth conditional distribution of the form $p(\tzz | \z) = C_p^{-1} \exp(\kappa \tzz^\top \zz)$. Let $h$ map onto a hypersphere with radius $\sqrt{\tau \kappa}$.\footnote{Note that in practice this can be implemented as a learnable rescaling operation of the network $f$.} Consider the conditional distribution $q_h$ parameterized by the model, as defined above in Theorem~\ref{thm:asym_inf_negatives_CE}, where the hypothesis class for $h$ is assumed to be sufficiently flexible such that $p(\tzz | \zz)$ and $\qh(\tzz|\zz)$ can match.
        If $h$ is a minimizer of the cross-entropy $\E_{p(\tzz | \zz)}[- \log \qh(\tzz | \zz)]$, then $p(\tzz|\zz) = \qh(\tzz | \zz)$ and $\forall \z, \tzz: \kappa \z^\top\tzz = h(\z)^\top h(\tzz)$.
    \end{proposition}
    \begin{proof}\vspace{-1.8\topsep}
        By assumption, $\qh(\tzz|\zz)$ is powerful enough to match $p(\tzz|\zz)$ for the correct choice of $h$ --- in particular, for $h(\z) = \sqrt{\tau \kappa} \z$. The global minimum of the cross-entropy between two distributions is reached if they match by value and have the same support.
        Thus, this means 
        \begin{equation}
            p(\tzz|\z) = \qh(\tzz|\z).
        \end{equation}
        This expression also holds true for $\tzz = \z$; additionally using that $h$ maps from a unit hypersphere to one with radius $\sqrt{\tau \kappa}$ yields
        \begin{align}
            \hspace{0.6cm} && p(\z|\z) &= \qh(\z|\z) \\
            \Leftrightarrow && C_p^{-1} e^{\kappa \z^\top \z} &= C_h(\z)^{-1} e^{h(\z)^\top h(\z)/\tau} \\
            \Leftrightarrow && C_p^{-1} e^{\kappa} &= C_h(\z)^{-1} e^{\kappa} \\
            \Leftrightarrow && C_p &= C_h.
        \end{align}
        As the normalization constants are identical we get for all $\zz, \tzz \in \Z$
        \begin{equation}
            e^{\kappa \z^\top \tzz} = e^{h(\z)^\top h(\tzz)} \Leftrightarrow \kappa \z^\top \tzz = h(\z)^\top h(\tzz).
        \end{equation}
    \end{proof}

    \begin{proposition}[Extension of the Mazur-Ulam theorem to hyperspheres and the dot product] \label{prop:extended_mazurulamspheres}
        Let $\Z = \Sp^{N-1}$ and $\Z' = \Sp^{N-1}_{r}$ be the hyperspheres with radius $1$ and $r > 0$, respectively. If $h: \R^N \to \Z'$ is differentiable in the vicinity of $\Z$ and its restriction to $\Z$ maintains the dot product up to a constant factor, i.e., $\forall \z, \tzz \in \Z: r^2 \z^\top \tzz = h(\z)^\top h(\tzz)$, then $h$ is an orthogonal linear transformation scaled by $r$ for all $\zz \in \Z$.
    \end{proposition}
    \begin{proof}\vspace{-1.8\topsep}
        First, we begin with the case $r = 1$.
        As $h$ maintains the dot product we have:
        \begin{align}
            \forall \zz, \tzz \in \Z: \zz^\top \tzz &= h(\zz)^\top h(\tzz).
            \intertext{
        We consider the partial derivative w.r.t. $\zz$ and obtain:
        }
            \forall \zz, \tzz \in \Z: \tzz &= \J_h^\top(\zz) h(\tzz).\label{eq:prop2-first-deriv}
            \intertext{Taking the partial derivative w.r.t. $\tzz$ yields}
            \forall \zz, \tzz \in \Z: \mathbf{I} &= \J_h^\top(\zz)\J_h(\tzz).
        \intertext{
        We can now conclude
        }
            \forall \zz, \tzz \in \Z:
            \J_h(\tzz)^{-1} &= \J_h^\top(\zz).
        \end{align}
        which implies a constant Jacobian matrix $\J_h(\zz) = \J_h$ as the identity holds on all points in $\Z$, and further that the Jacobian $\J_h$ is orthogonal.
        Hence, $\forall \zz \in \Z: h(\zz) = \J_h \zz$ is an orthogonal linear transformation.
        
        Finally, for $r \neq 1$ we can leverage the previous result by introducing $h'(\zz) :=  h(\zz) / r$. For $h'$ the previous argument holds, implying that $h'$ is an orthogonal transformation. Therefore, the restriction of $h$ to $\Z$ is an orthogonal linear transformation scaled by $r^2$.
    \end{proof}

    Taking all of this together, we can now prove Theorem~\ref{thm:ident_matching}:
    \vspace{\topsep}
    \begin{theorem} \label{thm:extended_ident_matching}
        Let $\Z = \Sp^{N-1}$, the ground-truth marginal be uniform, and the conditional a vMF distribution (cf. Eq.~\ref{eq:vmf_conditional}). Let the restriction of the mixing function $g$ to $\Z$ be injective and $h$ be differentiable in a vicinity of $\Z$. If the assumed form of $\qh$, as defined above, matches that of $p$, and if $f$ is differentiable and minimizes the CL loss as defined in Eq.~\eqref{eq:contrastive_loss}, then for fixed $\tau > 0$ and $M\to\infty$, $h = f \circ g$ is linear, i.e., $f$ recovers the latent sources up to an orthogonal linear transformation and a constant scaling factor.
    \end{theorem}
    \begin{proof}\vspace{-1.8\topsep}
        As $f$ minimzes the contrastive loss $\lcontr$ we can apply Theorem~\ref{thm:asym_inf_negatives_CE} to see that $f$ also minimizes the cross-entropy between  $p(\tzz | \zz)$ and $\qh(\tzz | \zz)$ for any point $\zz$ on $\Z$. This means, we can apply Proposition~\ref{prop:correct_model_ce_isometry} to show that the concatenation $h = f \circ g$ is an isometry with respect to the dot product. Finally, according to Proposition~\ref{prop:mazurulamspheres}, $h$ must then be a composition of an orthogonal linear transformation and a constant scaling factor. Thus, $f$ recovers the latent sources up to orthogonal linear transformations, concluding the proof.
    \end{proof}
    
\subsection{Extension of theory to subspaces of \texorpdfstring{$\mathbb{R}^N$}{RN}} \label{apx:rn_extension}
    Here, we show how one can generalize the theory above from $\Z = \Sp^{N-1}$ to $\Z \subseteq \mathbb{R}^N$. Under mild assumptions regarding the ground-truth conditional distribution $p$ and the model distribution $\qh$, we prove that all minimizers of the cross-entropy between $p$ and $\qh$ are linear functions, if $\Z$ is a convex body. Note that the hyperrectangle $[a_1, b_1] \times \ldots \times [a_N, b_N]$ is an example of such a convex body.
    
\subsubsection{Assumptions} \label{apx:rn_assumptions}
    First, we restate the core assumptions for this proof. The main difference to the assumptions for the hyperspherical case above is that we assume different conditional distributions: instead of rotation-invariant von Mises-Fisher distributions, we use translation-invariant distributions (up to restrictions determined by the finite size of the space) of the exponential family.
    
\paragraph{Generative process} \label{apx:rn_gt_assumptions}
    Let $g : \Z \to \X$ be an injective function between the two spaces $\Z \subseteq\mathbb{R}^N$ and $\X \subseteq \R^K$ with $K\geq N$ and where $\Z$ is a convex body (e.g., a hyperrectangle). Further, let the marginal distribution be uniform, i.e., $p(\z) = |\Z|^{-1}$. 
    We assume that the conditional distribution over positive pairs $p(\tzz | \z)$ is an exponential distribution
    \begin{equations}
        p(\tzz|\z) &= C_p^{-1}(\z) e^{-\lambda \delta(\tzz, \zz)} \\\text{with} \quad C_p(\z) :&= \int e^{-\lambda \delta(\z,\tzz)} \,\d\tzz ,
    \end{equations}
    where $\lambda > 0$ a parameter controlling the width of the distribution and $\delta$ is a (semi-)metric. If $\delta$ is a semi-metric, i.e., it does not fulfill the triangle inequality, there must exist a metric $\delta'$ such that $\delta$ can be written as the composition of a continuously invertible map $j:\R_{\geq0} \to \R_{\geq0}$ with $j(0) = 0$ and the metric, i.e., $\delta = j \circ \delta'$.
    Finally, we assume that during training one has access to samples from both of these distributions.
    
    Note that unlike for the hypersphere, when sampling positive pairs $\zz, \tzz \sim p(\zz)p(\tzz | \zz)$, it is no longer guaranteed that the marginal distributions of $\zz$ and $\tzz$ are the same. When referencing the density functions -- or using them in expectation values -- $p(\cdot)$ will always denote the same marginal density, no matter if the argument is $\zz$ or $\tzz$.
    Specifically, $p(\tzz)$ does not refer to $\int p(\zz) p(\tzz | \zz) d\zz$.

\paragraph{Model} \label{apx:rn_model_assumptions}
    Let $\Z'$ be a subset of $\mathbb{R}^N$ that is a convex body and let $f : \X \to \Z'$ be the model whose parameters are optimized. We associate a conditional distribution $\qh(\tzz|\zz)$ with our model $f$ through
    \begin{equations} \label{eq:rn_qhjoint}
        \qh(\tzz|\z) &= C_q^{-1}(\zz)e^{-\delta(h(\tzz), h(\zz))/\tau}\quad \\ \text{with} \quad C_q(\zz) :&= \int e^{-\delta(h(\tzz), h(\zz))/\tau} \,\d\tzz,
    \end{equations}
    where $C_q(\zz)$ is the partition function and $\delta$ is defined above.
    
\subsubsection{Minimizing the cross-entropy}\label{sec:upper_bound_cross_entropy}
    In a first step, we show the analogue of Proposition~\ref{prop:asym_inf_negatives} for $\Z$ being a convex body:
    \vspace{-\topsep}
    \begin{proposition} \label{lem:delta_contrastive_asym_inf_negatives}
        For fixed $\tau >0$, as the number of negative samples $M \to \infty$, the $\ldeltacontr$ loss converges to
        \begin{equations}
            \lim_{M \to \infty} \ldeltacontr(f; \tau, M) -\log M =  \\ 
            \quad\quad \ldeltaalign(f; \tau) + \ldeltaunif(f; \tau), 
        \end{equations}
        where
        \begin{equations}
             \ldeltaalign(f; \tau) &:=  \frac{1}{\tau} \expectunder{\substack{\z \sim p(\z) \\ \tzz \sim p(\tzz | \z)}}{\delta(h(\tzz), h(\z)))} \\
             \ldeltaunif(f;\tau) &:= \expectunder{\z \sim p(\z)}{\log\left( \expectunder{\tzz \sim p(\tzz)}{e^{-\delta(h(\tzz), h(\z))/\tau}}\right)},
        \end{equations}
        and $\ldeltacontr(f; \tau, M)$ is as defined in Eq.~\eqref{eq:delta_contrastive_loss}.
    \end{proposition}
    \begin{proof}\vspace{1.8\topsep}
        This proof is adapted from \citet{wang2020understanding}. By the Continuous Mapping Theorem and the law of large numbers, for any $\x, \txx$ and $\{\, \x^-_i \,\}_{i=1}^M$ it follows almost surely 
        \begin{equations}
            \lim_{M \to \infty} &\log \bigg( \frac{1}{M}e^{-\delta(f(\x), f(\txx))/\tau)} + \\ 
            &\quad\quad \frac{1}{M}\sum_{i=1}^M e^{-\delta(f(\x), f(\x^-_i))/\tau} \bigg) \\
            &= \log \left( \expectunder{\x^- \sim p_\mathsf{data}}{e^{-\delta(f(\x), f(\x^-))/\tau}} \right) \\
            &= \log \left( \expectunder{\tzz \sim p(\tzz)}{e^{-\delta(h(\z), h(\tzz))/\tau}} \right),
        \end{equations}
        where in the last step we expressed the sample $\x$ and negative examples $\x^-$ in terms of their latent factors.
        
        We can now express the limit of the entire loss function as
        \begin{equations}
            &\lim_{M\to \infty} \ldeltacontr(f; \tau, M) - \log M \\
            &= \frac{1}{\tau} \expectunder{(\x, \txx) \sim p_\mathsf{pos}}{\delta(f(\x), f(\txx))} \\
            &\quad + \lim_{M\to \infty} \underset{\substack{
                (\x, \tx) \sim p_\mathsf{pos} \\
                \{\xx^-_i\}_{i=1}^M \overset{\text{i.i.d.}}{\sim} p_\mathsf{data}
            }}{\mathbb{E}} \Bigg[ \log \bigg( \frac{1}{M}e^{-\delta(f(\x), f(\txx))/\tau} \\
            &\mkern150mu + \frac{1}{M}\sum_{i=1}^M e^{-\delta(f(\x), f(\x^-))/\tau} \bigg) \Bigg] \\
            &= \frac{1}{\tau}\expectunder{(\x, \txx) \sim p_\mathsf{pos}}{\delta(f(\x), f(\txx))} \\
            &\quad + \underset{\substack{
                (\x, \tx) \sim p_\mathsf{pos} \\
                \{\xx^-_i\}_{i=1}^M \overset{\text{i.i.d.}}{\sim} p_\mathsf{data}
            }}{\mathbb{E}} \Bigg[ \lim_{M\to \infty} \log \bigg( \frac{1}{M}e^{-\delta(f(\x), f(\txx))/\tau} \\
            &\mkern150mu + \frac{1}{M}\sum_{i=1}^M e^{-\delta(f(\x), f(\x^-_i))/\tau} \bigg) \Bigg].
        \end{equations}
        Note that as $\delta$ is a (semi-)metric, the expression $e^{-\delta(f(\x), f(\txx))}$ is upper-bounded by $1$. Hence, according to the Dominated Convergence Theorem one can switch the limit with the expectation value in the second step. 
        Inserting the previous results yields
        \begin{equations}
            &= \frac{1}{\tau}\expectunder{(\x, \txx) \sim p_\mathsf{pos}}{\delta(f(\x), f(\txx))} \\
            &\quad + \expectunder{\x \sim p_\mathsf{data}}{\log \left( \expectunder{\x^- \sim p_\mathsf{data}}{e^{-\delta(f(\x), f(\x^-))/\tau}} \right)} \\
            &= \frac{1}{\tau}\expectunder{\substack{\z \sim p(\z) \\ \tzz \sim p(\tzz | \z)}}{\delta(h(\z), h(\tzz))} \\
            &\quad + \expectunder{\z \sim p(\z)}{\log \left( \expectunder{\tzz \sim p(\tzz)}{e^{-\delta(h(\z), h(\tzz))/\tau}} \right)} \\
            &= \ldeltaalign(f; \tau) + \ldeltaunif(f; \tau).
        \end{equations}
    \end{proof}
    
    Next, we derive a property similar to Theorem~\ref{thm:asym_inf_negatives_CE}, which suggests a practical method to find minimizers of the cross-entropy between the ground-truth $p$ and model conditional $\qh$. This property is based on our previously introduced objective function in Eq.~(\ref{eq:delta_contrastive_loss}), which is a modified version of the InfoNCE objective in Eq.~(\ref{eq:contrastive_loss}).
    
    \vspace{\topsep}
    \begin{theorem} \label{thm:rn_delta_contrastive_CE}
        Let $\delta$ be a semi-metric and $\tau, \lambda > 0$ and let the ground-truth marginal distribution $p$ be uniform. Consider a ground-truth conditional distribution $p(\tzz | \zz) = C_p^{-1}(\z) \exp(-\lambda \delta(\tzz, \zz))$ and the model conditional distribution
        \begin{equations} \label{eq:extended_rn_qhjoint}
            \qh(\tzz | \zz) &= C_h^{-1}(\zz) e^{-\delta(h(\tzz), h(\zz))/\tau} \\
            \text{with} \quad C_h(\z) :&= \int_\Z e^{-\delta(h(\tzz), h(\zz))/\tau} \d\tzz.
        \end{equations}
        Then the cross-entropy between $p$ and $\qh$ is given by 
        \begin{equations}
            \lim_{M \to \infty} \ldeltacontr(f; \tau, M) - \log M + \log |\Z| = \\ \expectunder{\z \sim p(\z)}{H(p(\cdot | \z), \qh(\cdot | \z)}, 
        \end{equations}
        which can be implemented by sampling data from the accessible distributions.
    \end{theorem}
    \begin{proof}\vspace{-1.8\topsep}
        We use the definition of the cross-entropy to write
        \begin{align}
            &\expectunder{\z \sim p(\z)}{H(p(\cdot | \z), \qh(\cdot | \z)}\\
            &= -\expectunder{\z \sim p(\z)}{ \expectunder{\tzz \sim p(\tzz | \z)} {\log(\qh(\tzz | \z))}}.
            \intertext{We insert the definition of $\qh$ and get}
            &= -\expectunder{\z \sim p(\z)}{ \expectunder{\tzz \sim p(\tzz | \z)} {\log(C_h^{-1}(\z)) - \frac{1}{\tau}\delta(h(\tzz), h(\z)))}}\\
            &= \expectunder{\z \sim p(\z)}{ \expectunder{\tzz \sim p(\tzz | \z)} {\log(C_h(\z)) + \frac{1}{\tau}\delta(h(\tzz), h(\z)))}}.\\
            \intertext{As $C_h(\z)$ does not depend on $\tzz$ it can be moved out of the inner expectation value, yielding}
            &= \expectunder{\z \sim p(\z)}{ \frac{1}{\tau}\expectunder{\tzz \sim p(\tzz | \z)}{\delta(h(\tzz), h(\z)))} + \log(C_h(\z))},\\
            \intertext{which can be written as}
            &= \frac{1}{\tau} \expectunder{\substack{\z \sim p(\z) \\ \tzz \sim p(\tzz | \z)}}{\delta(h(\tzz), h(\z)))} + \expectunder{\z \sim p(\z)}{\log(C_h(\z))}.\\
            \intertext{Inserting the definition of $C_h$ gives}
            &= \frac{1}{\tau}\expectunder{\substack{\z \sim p(\z) \\ \tzz \sim p(\tzz | \z)}}{\delta(h(\tzz), h(\z)))} \\
            &\quad\quad +\expectunder{\z \sim p(\z)}{\log \left(\int e^{-\delta(h(\tzz), h(\z))/\tau}d\tzz\right)}.\\
            \intertext{Next, the second term can be expanded by $1=|\Z||\Z|^{-1}$, yielding}
            &= \frac{1}{\tau}\expectunder{\substack{\z \sim p(\z) \\ \tzz \sim p(\tzz | \z)}}{\delta(h(\tzz), h(\z)))} \\
            &\quad\quad +\expectunder{\z \sim p(\z)}{\log \left(\int \frac{|\Z|}{|\Z|} e^{-\delta(h(\tzz), h(\z))/\tau}d\tzz\right)}.\\
            \intertext{Finally, by using that the marginal is uniform, i.e., $p(\z) = |\Z|^{-1}$, this can be simplified as}
            &= \frac{1}{\tau}\expectunder{\substack{\z \sim p(\z) \\ \tzz \sim p(\tzz | \z)}}{\delta(h(\tzz), h(\z)))} \\
            &\quad\quad +\expectunder{\z \sim p(\z)}{\log\left( \expectunder{\tzz \sim p(\tzz)}{e^{-\delta(h(\tzz), h(\z))/\tau}}\right)} \\
            &\quad\quad + \log |\Z| \\
            &= \lim_{M\to \infty} \ldeltacontr(f; \tau, M) - \log M + \log p|\Z|.
        \end{align}
    \end{proof}
    
\subsubsection{Cross-entropy minimizers are isometries}
    Now we show a version of Proposition~\ref{prop:correct_model_ce_isometry}, that is generalized from hyperspherical spaces to (subsets of) $\mathbb{R}^N$.

    \vspace{\topsep}
    \begin{proposition}[Minimizers of the cross-entropy are isometries]\label{prop:rn_ce_minimzers_isometries}
        Let $\delta$ be a semi-metric. Consider the conditional distributions of the form $p(\tzz | \zz) = C_p^{-1}(\z) \exp(-\delta(\tzz, \zz)/\lambda)$ and
        \begin{equations}
            \qh(\tzz | \zz) &= C_h^{-1}(\zz) e^{-\delta(h(\tzz), h(\zz))/\tau} \\
            \text{with} \quad C_h(\z) :&= \int_\Z e^{-\delta(h(\tzz), h(\zz))/\tau} \d\tzz,
        \end{equations}
        where the hypothesis class for $h$ is assumed to be sufficiently flexible such that $p(\tzz | \zz)$ and $\qh(\tzz|\zz)$ can match for any point $\zz$.
        If $h$ is a minimizer of the cross-entropy $\lce = \E_{p(\tzz | \zz)}[- \log \qh(\tzz | \zz)]$, then $h$ is an isometry, i.e.,
        $\forall \zz, \tzz \in \Z: \lambda\tau \delta(\z, \tzz) = \delta(h(\z), h(\tzz))$. 
    \end{proposition}
    \begin{proof}\vspace{-1.8\topsep}
        Note that $\qh(\tzz|\z)$ is powerful enough to match $p(\tzz|\zz)$ for the correct choice of $h$, e.g. the identity. The global minimum of cross-entropy between two distributions is reached if they match by value and have the same support.
        Hence, if $p$ is a regular density, $\qh$ will be a regular density, i.e., $\qh$ is continuous and has only finite values $0 \leq \qh < \infty$.
        As the two distributions match, this means 
        \begin{equation}
            p(\tzz|\zz) = \qh(\tzz|\zz).
        \end{equation}
        This expression also holds true for $\tzz = \z$; additionally using the property $\delta(\z, \z) = 0$ yields
        \begin{align}
            \hspace{0.6cm} && p(\zz|\zz) &= \qh(\zz|\zz) \\
            \Leftrightarrow && C_p^{-1}(\z) e^{-\delta(\zz, \zz)/\lambda} &= C_h^{-1}(\zz) e^{-\delta(h(\zz), h(\zz))/\tau} \\
            \Leftrightarrow && C_p(\z) &= C_h(\zz).
        \intertext{As the normalization constants are identical, we obtain for all $\zz, \tzz \in \Z$}
            \hspace{0.6cm} && e^{-\delta(\tzz, \zz)/\lambda} &= e^{-\delta(h^*(\tzz), h^*(\zz))/\tau}\\ \Leftrightarrow && \delta(\tzz, \zz) &= \frac{\lambda}{\tau} \delta(h^*(\tzz),h^*(\zz)).
        \end{align}
        By introducing a new semi-metric $\delta' := \lambda\tau^{-1} \delta$, we can write this as $\delta(\tzz, \zz) = \delta'(h(\tzz),h(\zz))$, which shows that $h$ is an isometry. If there is no model mismatch, i.e., $\lambda = \tau$, this means $\delta(\z, \tzz) = \delta(h(\z), h(\tzz))$.
    \end{proof}
    Note, that this result does not depend on the choice of $\Z$ but just on the class of conditional distributions allowed.

\subsubsection{Cross-entropy minimization identifies the ground-truth factors} \label{apx:rn_ce_min_identifiability}

    Before we continue, let us recall a Theorem by \citet{mankiewicz1972extension}:
    \vspace{\topsep}
    \begin{customtheorem}{C}[\citealp{mankiewicz1972extension}]\label{thm:mankiewicz}
        Let $\mathcal{X}$ and $\mathcal{Y}$ be normed linear spaces and let $\mathcal{V}$ be a convex body in $\mathcal{X}$ and $\mathcal{W}$ a convex body in $\mathcal{Y}$. Then every surjective isometry between $\mathcal{V}$ and $\mathcal{W}$ can be uniquely extended to an affine isometry between $\mathcal{X}$ and $\mathcal{Y}$.
    \end{customtheorem}
    \begin{proof}\vspace{-1.8\topsep}
        See \citet{mankiewicz1972extension}.
    \end{proof}
    
    In addition, it is known that isometries on closed spaces are bijective:
     \begin{customlemma}{A} \label{lem:compact_space_isometry_bijective}
        Assume $h$ is an isometry of the closed space $\Z$ into itself, i.e., $\forall \z, \tzz: \delta(\z, \tzz) = \delta(h(\z), h(\tzz))$. Then $h$ is bijective.
    \end{customlemma}
    \begin{proof}
        See Lemma (2.6) in \citet{calka1982local} for surjectivity.
        We show the injectivity by contradiction. Assume $h$ is not injective. Then we can find a point $\tzz \neq \zz$ where $h(\zz) = h(\tzz)$. But then $\delta(\z, \tzz) > \delta(\z, \z)$ and $\delta(h(\zz), h(\tzz)) = \delta(h(\zz), h(\zz)) = 0$ by the properties of $\delta$. Hence, $h$ is injective.
    \end{proof}
    
    Before continuing, we need to generalize the class of functions we consider as distance measures:
    \begin{lemma}\label{lem:semimetric}
        Let $\delta'$ be a the composition of a continuously invertible function $j: \R_{\geq0} \to \R_{\geq0}$ with $j(0) = 0$ and a metric $\delta$, i.e., $\delta' := j \circ \delta$.
        Then,
        (i) $\delta'$ is a semi-metric and (ii) if a function $h: \R^n \to \R^n$ is an isometry of a space with the semi-metric $\delta'$, it is also an isometry of the space with the metric $\delta$.
    \end{lemma}
        \begin{proof}
            (i) Let $\z, \tzz \in \Z$. Per assumption $j$ must be strictly monotonically increasing on $\R_{\geq0}$. Since $\delta$ is a metric it follows $\delta(\z, \tzz) \geq 0 \Rightarrow \delta'(\z, \tzz) = j(\delta(\z, \tzz)) \geq 0$, with equality iff $\z = \tzz$. Furthermore, since $\delta$ is a metric it is symmetric in its arguments and, hence, $\delta'$ is symmetric in its arguments. Thus, $\delta'$ is a semi-metric.
            
            (ii) $h$ is an isometry of a space with the semi-metric $\delta'$, allowing to derive that for all $\z, \tzz \in \Z$,
            \begin{align}
                \delta'(h(\z), h(\tzz)) &= \delta'(\z, \tzz)\\
                j(\delta(h(\z), h(\tzz))) &= j(\delta(\z, \tzz))\\
                \intertext{and, applying the inverse $j^{-1}$ which exists by assumption, yields}
                \delta(h(\z), h(\tzz)) &= \delta(\z, \tzz),
            \end{align}
            concluding the proof.
        \end{proof}
    
    By combining the properties derived before we can show that $h$ is an affine function:
    \vspace{\topsep}
    \begin{theorem} \label{thm:ce_rn_linear_identifiable}
        Let $\Z = \Z'$ be a convex body in $\mathbb{R}^N$.
        Let the mixing function $g$ be differentiable and invertible. If the assumed form of $\qh$ as defined in Eq.~\eqref{eq:rn_qhjoint} matches that of $p$, and if $f$ is differentiable and minimizes the cross-entropy between $p$ and $\qh$, then we find that $h = f \circ g$ is affine, i.e., we recover the latent sources up to affine transformations.
    \end{theorem}
    \begin{proof}\vspace{-1.8\topsep}
        According to Proposition~\ref{prop:rn_ce_minimzers_isometries} $h$ is an isometry and $\qh$ is a regular probability density function. If the distance $\delta$ used in the conditional distributions $p$ and $\qh$ is a semi-metric as in Lemma~\ref{lem:semimetric}, it follows that $h$ is also an isometry for a proper metric. This also means that $h$ is bijective according to Lemma~\ref{lem:compact_space_isometry_bijective}.
        Finally, Theorem~\ref{thm:mankiewicz} says that $h$ is an affine transformation.
    \end{proof}
    
    We use the assumption that the marginal $p(\z)$ is uniform, to show
    \vspace{\topsep}
    \begin{theorem} \label{thm:extended_rn_linear_identifiable}
        Let $\Z$ be a convex body in $\mathbb{R}^N$, $h = f\circ g:\Z\to\Z$, and $\delta$ be a metric or a semi-metric as defined in Lemma~\ref{lem:semimetric}. Further, let the ground-truth marginal distribution be uniform and the conditional distribution be as \eqref{eq:rn_conditional}. Let the mixing function $g$ be differentiable and injective. If the assumed form of $\qh$ matches that of $p$, i.e., 
        \begin{equations}
            \qh(\tzz|\z) &= C_q^{-1}(\zz)e^{-\delta(h(\tzz), h(\zz))/\tau}\quad \\ \text{with} \quad C_q(\zz) :&= \int e^{-\delta(h(\tzz), h(\zz))/\tau} \,\d\tzz,
        \end{equations}
        and if $f$ is differentiable and minimizes the $\ldeltacontr$ objective in \eqref{eq:delta_contrastive_loss} for $M \to \infty$, we find that $h = f \circ g$ is invertible and affine, i.e., we recover the latent sources up to affine transformations.
    \end{theorem}
    \begin{proof}\vspace{-1.8\topsep}
        According to Theorem~\ref{thm:rn_delta_contrastive_CE} $h$ minimizes the cross-entropy between $p$ and $\qh$ as defined in Eq.~\eqref{eq:qhjoint}. Then according to Theorem~\ref{thm:ce_rn_linear_identifiable}, $h$ is an affine transformation.
    \end{proof}
    This result can be seen as a generalized version of Theorem~\ref{thm:ident_matching}, as it is valid for any convex body $\Z \subseteq \mathbb{R}^N$ and allows a larger variety of conditional distributions. A missing step is to extend this theory beyond uniform marginal distributions. This will be addressed in future work.
    
    Under some assumptions we can further narrow down possible forms of $h$, thus, showing that $h$ in fact solves the nonlinear ICA problem only up to permutations and elementwise transformations.

    For this, let us first repeat a result from \citet{li1994isometries}, that shows an important property of isometric matrices:
    \vspace{\topsep}
    \begin{customtheorem}{D} \label{prop:lp_neq_2_isometry_permutation}
        Suppose $1 \leq \alpha \leq \infty$ and $\alpha \neq 2$. An $n \times n$ matrix $\A$ is an isometry of $L^\alpha$-norm if and only if $\A$ is a generalized permutation matrix, i.e., $\forall \z: (\A \z)_i = \alpha_i \z_{\sigma(i)}$, with $\alpha_i = \pm 2$ and $\sigma$ being a permutation.
    \end{customtheorem}
    \begin{proof}\vspace{-1.8\topsep}
        See \citet{li1994isometries}. Note that this can also be concluded from the Banach-Lamperti Theorem \citep{lamperti1958isometries}.
    \end{proof}
    
    Leveraging this insight, we can finally show:
    \vspace{\topsep}
    \begin{theorem} \label{thm:extended_rn_permutation_identifiable}
        Let $\Z$ be a convex body in $\mathbb{R}^N$, $h: \Z \to \Z$, and $\delta$ be an $L^\alpha$ metric for $\alpha \geq 1, \alpha \neq 2$ or the $\alpha$-th power of such an $L^\alpha$ metric. Further, let the ground-truth marginal distribution be uniform and the conditional distribution be as in Eq.~\eqref{eq:rn_conditional}, and let the mixing function $g$ be differentiable and invertible. If the assumed form of $\qh(\cdot|\z)$ matches that of $p(\cdot|\z)$, i.e., both use the same metric $\delta$ up to a constant scaling factor, and if $f$ is differentiable and minimizes the $\ldeltacontr$ objective in Eq.~\eqref{eq:delta_contrastive_loss} for $M \to \infty$ we find that $h = f \circ g$ is a composition of input independent permutations, sign flips and rescalings.
    \end{theorem}
    \begin{proof}\vspace{-1.8\topsep}
        First, we prove the case where both conditional distributions use exactly the same metric. By Theorem~\ref{thm:extended_rn_linear_identifiable} $h$ is an affine transformation. Moreover, according to Proposition~\ref{prop:rn_ce_minimzers_isometries} is an isometry. Thus, by %
        Theorem~\ref{prop:lp_neq_2_isometry_permutation}, $h$ is a
        generalized permutation matrix, i.e., a composition of permutations and sign flips.
        
        Finally, for the case that $\delta$ matches the similarity measure in the ground-truth conditional distribution defined in Eq.~\eqref{eq:rn_conditional} (denoted as $\delta^*$) only up to a constant rescaling factor $r$, we know
        \begin{equations}
            &\mkern28mu \forall \z, \tzz: \delta^*(\z, \tzz) = \delta(h(\z), h(\tzz)) \\
            &\Leftrightarrow \delta^*(\z, \tzz) = \delta^*\left(\frac{1}{r}h(\z), \frac{1}{r}h(\tzz)\right).
        \end{equations}
        Thus, $\frac{1}{r}h$ is a $\delta^*$ isometry and the same argument as above holds, concluding the proof.
    \end{proof}

    \begin{table*}[htp]
        \centering
        \caption{Identifiability up to affine transformations on the training set of 3DIdent. Mean $\pm$ standard deviation over 3 random seeds. As earlier, only the first row corresponds to a setting that matches the theoretical assumptions for linear identifiability; the others show distinct violations. Supervised training with unbounded space achieves scores of $R^2=(99.98 \pm 0.01)$\% and $\text{MCC}=(99.99 \pm 0.01)$\%. The last row refers to using the SimCLR \citep{chen2020simple} augmentations to generate positive pairs. The last row refers to using the image augmentations suggested by \citet{chen2020simple} to generate positive image pairs; for details see Sec.~\ref{apx:experiment_details}.
        In contrast to Table~\ref{tab:3dident_results_test}, the scores here are reported on the same data the models were trained on.
        }
        \small
        \begin{tabular}{ccccccc}
            \toprule
            Dataset & \multicolumn{3}{c}{Model $f$} & Identity [\%] & \multicolumn{2}{c}{Unsupervised [\%]} \\
            $p(\cdot|\cdot)$ & Space & $\qh(\cdot|\cdot)$ & M. & $R^2$ & $R^2$ & MCC \\
            \midrule
            
            Normal & Box & Normal & \cmark & $5.35 \pm 0.72$  & $97.83 \pm 0.13$ & $98.85 \pm 0.07$\\            

            Normal & Unbounded & Normal & \xmark & \dittotikz & $97.72 \pm 0.02$ & $55.90 \pm 2.22$\\

            Laplace & Box & Normal & \xmark & \dittotikz & $97.95 \pm 0.05$ & $98.94 \pm 0.03$\\
            
            Normal & Sphere & vMF & \xmark & \dittotikz & $66.73 \pm 0.03$ & $42.72 \pm 3.20  $\\
            
            Augm. & Sphere & vMF & \xmark & \dittotikz & $45.94 \pm 1.80$ & $47.6 \pm 1.45$\\

            \bottomrule
        \end{tabular}
        \label{tab:3dident_results_train}
    \end{table*}

\subsection{Experimental details} \label{apx:experiment_details}
    For the experiments presented in Sec.~\ref{sec:toy_experiments} we train our feature encoder for $300\,000$ iterations with a batch size of $6144$ utilizing Adam \citep{kingma2014adam} with a learning rate of $10^{-4}$. Like \citet{hyvarinen2016unsupervised,hyvarinen2017nonlinear}, for the mixing network, we i) use $0.2$ for the angle of the negative slope\footnote{See e.g. \url{https://pytorch.org/docs/stable/generated/torch.nn.LeakyReLU.html}}, ii) use $L^2$ normalized weight matrices with minimum condition number of $25\,000$ uniformly distributed samples. For the encoder, we i) use the default ($0.01$) negative slope ii) use $6$ hidden layers with dimensionality [$N\cdot10$, $N\cdot50$, $N\cdot50$, $N\cdot50$, $N\cdot50$, $N\cdot10$] and iii) initialize the normalization magnitude as $1$. We sample $4096$ latents from the marginal for evaluation. For MCC~\citep{hyvarinen2016unsupervised,hyvarinen2017nonlinear} we use the Pearson correlation coefficient\footnote{See e.g. \url{https://numpy.org/doc/stable/reference/generated/numpy.corrcoef.html}}; we found there to be no difference with Spearman\footnote{See e.g. \url{https://docs.scipy.org/doc/scipy/reference/generated/scipy.stats.spearmanr.html}}.
    
    For the experiments presented in Sec.~\ref{sec:kitti_experiments}, we use the same architecture as the encoder in \citep{klindt2020slowvae}. As in~\citep{klindt2020slowvae}, we train for $300\,000$ iterations with a batch size of $64$ utilizing Adam \citep{kingma2014adam} with a learning rate of $10^{-4}$. For evaluation, as in~\citep{klindt2020slowvae}, we use $10\,000$ samples and the Spearman correlation coefficient.
    
    For the experiments presented in Sec.~\ref{sec:3dident_experiments}, we train the feature encoder for $200\,000$ iterations using Adam with a learning rate of $10^{-4}$. For the encoder we use a ResNet18 \citep{he2015deep} architecture followed by a single hidden layer with dimensionality $N\cdot10$ and LeakyReLU activation function using the default ($0.01$) negative slope. The scores on the training set are evaluated on $10\%$ of the whole training set, $25\,000$ random samples. The test set consists of $25\,000$ samples not included in the training set.
    For the last row of Tab.~\ref{tab:3dident_results_test} and Tab.~\ref{tab:3dident_results_train} we used the best-working combination of image augmentations found by \citet{chen2020simple} to sample positive pairs. To be precise, we used a random crop and resize operation followed by a color distortion augmentation. The random crops had a uniformly distributed size (between 8\% and 100\% of the original image area) and a random aspect ration (between 3/4 and 4/3); subsequently, they were resized to the original image dimension ($224 \times 224$) again. The color distortion operation itself combined color jittering (i.e., random changes of the brightness, contrast, saturation and hue) with color dropping (i.e., random grayscale conversations). We used the same parameters for these augmentations as recommended by \citet{chen2020simple}. %
    
    The experiments in Sec.~\ref{sec:toy_experiments} took on the order of 5-10 hours on a GeForce RTX 2080 Ti GPU, the experiments on KITTI Masks took 1.5 hours on a GeForce RTX 2080 Ti GPU and those on 3DIdent took 28 hours on four GeForce RTX 2080 Ti GPUs. The creation of the 3DIdent dataset additionally required approximately 150 hours of compute time on a GeForce RTX 2080 Ti.
    
\subsection{Details on 3DIdent} \label{apx:3dident_details}
    We build on the rendering pipeline of \citet{johnson2017clevr} and use the Blender engine \citep{blender}, as of version 2.91.0, for image rendering.
    The scenes depicted in the dataset show a rotated and translated object onto which a spotlight is directed. The spotlight is located on a half-circle above the scene and shines down. The scenes can be described by $10$ parameters:
    the position of the object along the X-, Y- and Z-axis, the rotation of the object described by Euler angles (3), the position of the spotlight described by a polar angle, and the hue of the object, the ground and the spotlight. The value range is $[-3,3]$ for all position parameters, and is $[-\pi/2, \pi/2]$ for the remaining parameters. The parameters are sampled from a $10$-dimensional unit hyperrectangle, then rescaled to their corresponding value range. This ensures that the variance of the latent factors is the same for all latent dimensions.
    
    To ensure that the generative process is injective, we take two measures: First, we use a non-rotationally symmetric object (Utah tea pot, \citealp{newell1975utah}), thus the rotation information is unambiguous. Second, we use different levels of color saturation for the object, the spotlight and the ground ($1.0$, $0.8$ and  $0.6$, respectively), thus the object is always distinguishable from the ground.
    
\subsubsection{Comparison to existing datasets} \label{apx:3dident_comparison}
    The proposed dataset contains high-resolution renderings of an object in a 3D scene. It features some aspects of natural scenes, e.g. complex 3D objects, different lighting conditions and continuous variables. Existing benchmarks~\citep{klindt2020slowvae, 3dshapes18, gondal2019transfer, dittadi2021transfer} for disentanglement in 3D scenes differ in important aspects to 3DIdent.
    
    KITTI Masks~\citep{klindt2020slowvae} only enables evaluating identification of the two-dimensional position and scale of the object instance. In addition, the observed  segmentation masks are significantly lower resolution than examples in our dataset. 3D Shapes~\citep{3dshapes18} and MPI3D~\citep{gondal2019transfer} are rendered at the same resolution ($64\times64$) as KITTI Masks. Whereas the dataset contributed by~\citep{dittadi2021transfer} is rendered at $2\times$ that resolution ($128\times128$), our dataset is rendered at $3.5\times$ that resolution ($224\times224$), the resolution at which natural image classification is typically evaluated~\citep{imagenet_cvpr09}. With that being said, we do note that KITTI Masks is unique in containing frames of natural video, and we thus consider it complementary to 3DIdent. 
    
    \citet{3dshapes18}, \citet{dittadi2021transfer}, and \citet{gondal2019transfer} contribute datasets which contain variable object rotations around one, one, and two rotation axes, respectively, while 3DIdent contains variable object rotation around all three rotation axes as well as variable lighting conditions. Furthermore, each of these datasets were generated by sampling latent factors from an equidistant grid, thus only covering a limited number values along each axis of variation, effectively resulting in a highly coarse discretization of naturally continuous variables. As 3DIdent instead samples the latent factors uniformly in the latent space, this better reflects the continuous nature of the latent dimensions.

\subsection{Effects of the Uniformity Loss}
    In previous work, \citet{wang2020understanding} showed that a part of the contrastive (InfoNCE) loss --- the uniformity loss --- effectively ensures that the encoded features are uniformly distributed over a hypersphere.
    We now show that this part is crucial to ensure that the mapping is bijective. More precisely, we demonstrate that if the distribution of the encoded/reconstructed latents $h(\z)$ has the same support as the distribution of $\z$, and both distributions are regular, i.e., their densities are non-zero and finite, then the transformation $h$ is bijective.
    
    First, we focus on the more general case of a map between manifolds:
    \vspace{\topsep}
    \begin{proposition}\label{apx:prop-uniform-biject}
        Let $\mathcal{M},\mathcal{N}$ be simply connected and oriented $\mathcal{C}^1$ manifolds without boundaries and $h:\mathcal{M}\to\mathcal{N}$ be a differentiable map. Further, let the random variable $\z \in \mathcal{M}$ be distributed according to $\z \sim p(\z)$ for a regular density function $p$, i.e., $0 < p < \infty$. If the pushforward $p_{\#h}(\z)$ of $p$ through $h$ is also a regular density, i.e., $0 < p_{\#h} < \infty$, then $h$ is a bijection.
    \end{proposition}
    \begin{proof}\vspace{-1.8\topsep}
        We begin by showing by contradiction that the Jacobian determinant of $h$ does not vanish, i.e., $|\det J_h| > 0$:
        
        Suppose that the Jacobian determinant $|\det J_h|$ vanishes for some $\z \in\mathcal{M}$. Then the inverse of the Jacobian determinant goes to infinity at this point and so does the density of $h(\z)$ according to the well-known transformation of probability densities. By assumption, both $p$ and $p_{\#h}$ must be regular density functions and, thus, be finite. This contradicts the initial assumption and so the Jacobian determinant $|\det J_h|$ cannot vanish.

        Next, we show that the mapping $h$ is proper. Note that a map is called proper if pre-images of compact sets are compact \citep{ruzhansky2015global}. Firstly, a continuous mapping between $\mathcal{M}$ and $\mathcal{N}$ is also closed, i.e., pre-images of closed subsets are also closed \citep{lee2013smooth}. In addition, it is well-known that continuous functions on compact sets are bounded. Lastly, according to the Heine–Borel theorem, compact subsets of $\mathbb{R}^D$ are closed and bounded. Taken together, this shows that $h$ is proper.
   
        Finally, according to Theorem~2.1 in \citep{ruzhansky2015global} a proper $h$ with non-vanishing Jacobian determinant is bijective, concluding the proof.
    \end{proof}
    This theorem directly applies to the case of hyperspheres, which are simply connected and oriented manifolds without boundary. This yields:
    \vspace{\topsep}
    \begin{corollary}
        Let $\Z$ be a hypersphere and $h:\mathcal{Z}\to\mathcal{Z}$ be a differentiable map. Further, let the marginal distribution $p(\z$) of the variable $\zz \in \Z$ be a regular density function, i.e., $0 < p < \infty$. If the pushforward $p_{\#h}$ of $p$ through $h$ is also a regular density, i.e., $0 < p_{\#h} < \infty$, then $h$ is a bijection.
    \end{corollary}
    Therefore, we can conclude that a loss term ensuring that the encoded features are distributed according to a regular density function, such as the uniformity term, makes the map $h$ bijective and prevents an information loss. Note that this does not assume that the marginal distribution of the ground-truth latents $p(\z)$ is uniform but only that it is regular and non-vanishing.
    
    Note that while the proposition shows that the uniformity loss is sufficient to ensure bijectivity, we can construct counterexamples if its assumptions (like differentiability) are violated even in just a single point.
    For instance, the requirement of $h$ being fully differentiable is most likely violated in large unregularized neural networks with ReLU nonlinearities.
    Here, one might need the full contrastive loss to ensure bijectivity of $h$.
    
\section*{ArXiv Changelog}

\begin{itemize}
    \item Current Version: Thanks to feedback from readers, we fixed a few inconsistencies in our notation. We also added a considerably simplified proof for Proposition~\ref{prop:mazurulamspheres}. 
    \item \href{https://arxiv.org/abs/2102.08850v3}{June 21, 2021}: We studied violations of the uniformity assumption in greater details, and added Figure~\ref{fig:uniformity_violation}. We thank the anonymous reviewers at ICML for their suggestions. This is also the version available in the proceedings of ICML 2021.
    \item \href{https://arxiv.org/abs/2102.08850v2}{May 25, 2021:} Extensions of the theory: We added additional propositions for the effects of the uniformity loss. 
    \item \href{https://arxiv.org/abs/2102.08850v1}{February 17, 2021}: First pre-print.
\end{itemize}

\end{document}